%% file: efficient-reasoning.tex
\documentclass{article}

\usepackage{microtype}
\usepackage{graphicx}
\usepackage{subfigure}
\usepackage{wrapfig}
\usepackage{booktabs} 
\usepackage{bbm}
\usepackage{makecell}
\usepackage{amsmath}
\DeclareMathOperator*{\Var}{Var}
\newcommand{\1}{\mathbb{1}}

\usepackage{amsthm}

\usepackage{hyperref}

\usepackage[main, final]{neurips_2025}
\usepackage[T1]{fontenc}
\usepackage{amsmath}
\usepackage{amssymb}
\usepackage{mathtools}
\usepackage{amsthm}

\usepackage[capitalize,noabbrev]{cleveref}

\theoremstyle{plain}
\newtheorem{theorem}{Theorem}[section]
\newtheorem{proposition}[theorem]{Proposition}
\newtheorem{lemma}[theorem]{Lemma}

\theoremstyle{definition}

\newtheorem{assumption}[theorem]{Assumption}
\theoremstyle{remark}

\DeclareMathOperator{\clip}{clip}
\usepackage[textsize=tiny]{todonotes}

\title{Training Language Models to Reason Efficiently}

\author{%
Daman Arora \\\
  Carnegie Mellon University\\
  \texttt{damana@andrew.cmu.edu} \\
  \And
  Andrea Zanette \\
  Carnegie Mellon University \\
  \texttt{zanette@cmu.edu} \\
}

\begin{document}

\maketitle

\vskip 0.3in




\begin{abstract}
\input{1-abstract}
\end{abstract}

\input{2-body}

\bibliography{efficient-reasoning}
\bibliographystyle{neurips_2025}

\input{appendix}
\input{checklist}

\end{document}

%% file: 1-abstract.tex
Scaling model size and training data has led to great advances in the performance of Large Language Models (LLMs). However, the diminishing returns of this approach necessitate alternative methods to improve model capabilities, particularly in tasks requiring advanced reasoning. Large reasoning models, which leverage long chain-of-thoughts, bring unprecedented breakthroughs in  problem-solving capabilities but at a substantial deployment cost associated to longer generations. Reducing inference costs is crucial for the economic feasibility, user experience, and environmental sustainability of these models.

In this work, we propose to train large reasoning models to reason efficiently. Our method incentivizes models to minimize unnecessary computational overhead while largely maintaining accuracy, thereby achieving substantial deployment efficiency gains. It  enables the derivation of a family of reasoning models with varying efficiency levels, controlled via a single hyperparameter. Experiments on two open-weight large reasoning models demonstrate significant reductions in inference cost while preserving most of the accuracy. 
 

%% file: 2-body.tex
\section{Introduction}
\label{sec:introduction}
Large language models (LLMs) have made significant advancements by pre-training larger models with extensive datasets \citep{kaplan2020scaling}, but this approach faces diminishing returns due to limited high-quality training data.
An alternative to improve model capabilities, especially in domains involving careful reasoning, involves allowing models to ``think" before answering, as seen in frontier reasoning models like OpenAI's o1 \citep{openai2024openaio1card}, Gemini 2.0 Flash Thinking Experimental, and DeepSeek-R1 \citep{guo2025deepseek}. These models produce intermediate tokens during inference, collectively referred to as \emph{chain-of-thoughts} \citep{wei2022chain}, to perform additional computations before returning an answer. The process of generating a long chain of thought before answering the user query is called \emph{reasoning}.
More precisely, \emph{large reasoning models} with chain-of-thoughts capable of performing advanced reasoning emerge from RL \citep{sutton2018reinforcement, guo2025deepseek} on base models using ground-truth scoring functions (e.g., correctness on math problems). 



\begin{figure*}[ht!]
\centering
  \includegraphics[width=0.95\textwidth]{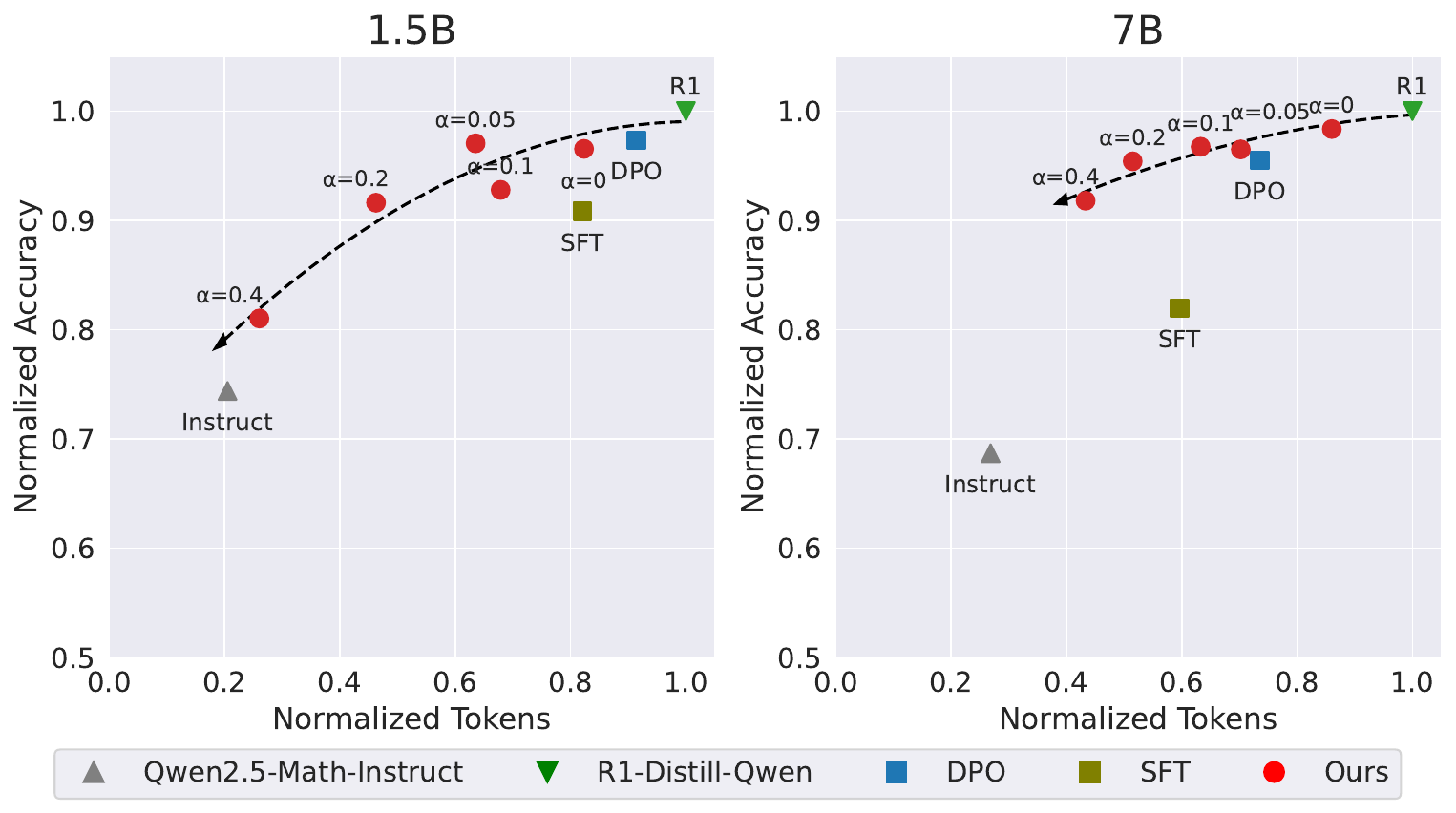}
  \caption{This figure describes the results of our training procedure aggregated on 5 evaluation datasets on both 1.5B and 7B scales. The x-axis represents response length normalized by the performance of the R1-Distilled model and the y-axis represents accuracy normalized by the accuracy of the R1-Distilled model. Each value of $\alpha$ has been aggregated over 3 seeds. Our training reward design allows us to easily and effectively navigate the token-accuracy tradeoff curve. Increasing the value of $\alpha$ generally results in models that have lower response lengths.}
  \label{fig:master_summary}
\end{figure*}

These reasoning models use test-time compute in the form of very long chain-of-thoughts, an approach that incur a high inference cost due to the quadratic cost of the attention mechanism and linear growth of the KV cache for transformer-based architectures \cite{vaswani2017attention}. However, effective deployment of LLMs demands models that are not only capable but also computationally efficient to serve. 
Even for resource-rich organizations such as large tech companies that have the resources to train reasoning models, excessive inference costs may mean operating at a loss rather than at a profit in order to match the competitor's offering.
Furthermore, reducing inference costs often reduces latency, improves responsiveness, and therefore improves user experience. Finally, lowering the inference computation positively benefits the environment by reducing carbon emissions.


To achieve this goal, we use policy gradient methods \citep{sutton2018reinforcement} to train the model to use lower number of tokens to reach the correct solution, thereby minimizing inference costs, ideally without compromising on accuracy. 
We use a modified reinforcement learning reward function which \emph{encourages the model to produce correct answers with short chain-of-thoughts}.
At the time of the initial publication of our pre-print on arXiv, we were among the first to consider training the model to be efficient at inference time, and
we discuss concurrent as well as related  literature in \cref{sec:relatedwork}. 
Our method allows the user to  \emph{control} the reduction in inference compute by adjusting a scalar coefficient in an intuitive way. In other words, 
starting from a reasoning model, our procedure allows us to derive a \emph{family} of models, each with increased generation efficiency (i.e., shorter chain-of-thoughts). 

We perform numerical experiments on two recently released open-weight large reasoning models, DeepSeek-R1-Distill-Qwen-1.5B and DeepSeek-R1-Distill-Qwen-7B \cite{guo2025deepseek}
and derive models with a substantial reduction in reasoning cost while approximately maintaining accuracy, see \cref{fig:master_summary} for a summary of our results. We observe that our training procedure allows us to gracefully navigate the compute-performance tradeoff curve, that is, reduce compute significantly with minimal loss in performance. For instance, for the 7B model, we can achieve a reduction of 50\% in tokens with less than 5\% reduction in accuracy.

We also observe that the reduction in response length depends on the hardness of the problem. For the 7B model, our method can reduce token usage by 16\% on the competition-level benchmark: American Invitational Mathematics Examination 2024 with a drop of 3.3\% points in accuracy, up to 37\% tokens on the MATH500 \citep{hendrycks2021measuring} dataset with a drop of 2.2\% in accuracy, and up to 65\% tokens on the GSM8K \citep{cobbe2021training} dataset with a drop of 1.7\% in accuracy. 


Our work only requires a couple of lines of changes to any standard reinforcement learning implementation. Beyond its simplicity, an attractive property of our approach is its \emph{computational efficiency}:
although training reasoning models
with large scale RL
has a prohibitive cost \cite{guo2025deepseek},
our procedure shows that training them to reason efficiently is highly viable
even with modest academic resources: 
our models are obtained with only 100 RL steps ($\sim$ 200 gradient updates). 
The fact that we achieve a performance comparable to that of the original reasoning model with a short training is surprising because in a few RL steps the model learns to optimize for shorter, more efficient reasoning patterns compared to the original model. All of our code and trained models are public at \url{https://github.com/Zanette-Labs/efficient-reasoning}. 

\section{Related Work}
\label{sec:relatedwork}

\textbf{Improving model capabilities with test-time compute} Several techniques such as Chain of Thoughts \citep{wei2022chain}, Self Consistency \citep{wang2022self}, Best-of-N, Monte Carlo Tree Search \citep{silver2017mastering}, Tree-of-thoughts \citep{yao2024tree}, Stream-of-Search \citep{gandhi2024stream}, Graph-of-Thoughts \citep{besta2024graph}, Process Reward Models \citep{lightman2024lets} and Self-Correction \citep{kumar2024training, welleck2023generating} have been proposed to improve performance by spending more inference time compute. It has also been shown that in some cases scaling inference compute can be more effective than scaling model size \citep{snell2025scaling, liu20251bllmsurpass405b}.  While the above mentioned techniques 
can be effective in specialized scenarios, 
modern large scale reasoning models
are trained with large scale RL with verifiable rewards and perform autoregressive generation.

\textbf{Large Reasoning Models}
Frontier reasoning such as OpenAI o1, Deepseek R1 and QwQ-Preview rely on long, monolithic chain-of-thoughts to perform advanced reasoning.   
They are trained with large scale reinforcement learning \cite{guo2025deepseek}, which leads them to develop abilities such as branching, verification and backtracking. 
Our approach aims to make these models more efficient.

\textbf{Efficient Serving}
While we focus on developing reasoning models that can be served efficiently, our approach is orthogonal to existing methods from the literature of efficient LLMs; see \citet{zhou2024surveyefficientinferencelarge} for a recent survey. 
For example, system-level techniques
build a system to accelerate inference. Some examples include speculative decoding \cite{leviathan2023fast} and batch engines like vLLM \cite{kwon2023efficient}; both can be directly combined with our method. Model-based techniques, on the other hand, act directly on the model to accelerate inference. Some examples include weight pruning \cite{liu2018rethinking} and quantization \cite{lin2024awq}, which are all complementary and can combined with our methodology.


\textbf{Concurrent works} \citet{chen2024think23overthinkingo1like} investigate the overthinking phenomena and propose methods to mitigate it by using heuristics such as First-Correct Solutions (FCS) and Greedy Diverse Solutions (GDS) to generate preference data which is then used for offline policy optimization. However, this method doesn't allow easily tuning the model to the user's compute budget. At the time of our pre-print coming online on arXiv, several papers with similar objectives were released as well. For instance, the Kimi k1.5 \citep{kimiteam2025kimik15scalingreinforcement} also report a method to shorten the chain-of-thought using a length penalty in the reward function while doing online RL, a procedure similar in principle but not identical to ours.
We note that their procedure does not appear to have a tunable parameter which allows to obtain a family of models--each with varying trade-offs--as we do. Another work in this direction is O1-Pruner \citep{luo2025o1prunerlengthharmonizingfinetuningo1like} which proposes a slightly different offline RL objective to minimize tokens while maintaining accuracy while we use online RL. Subsequent to our pre-print, multiple works have tried to look into this research direction. For instance, \cite{qu2025optimizingtesttimecomputemeta} pose the problem of minimizing compute usage as a meta-RL problem and use online RL to solve it. \cite{aggarwal2025l1controllinglongreasoning} train reasoning models to follow exact token constraints specified in prompts. \cite{xia2025tokenskipcontrollablechainofthoughtcompression} prune tokens from Chain-of-Thoughts using semantic importance and perform SFT to obtain models with controllable compression.

\textbf{Efficiency of Chain-of-Thought} \citet{jin-etal-2024-impact} find that lengthening chain-of-thought has a correlation with improving performance. Conditional training as done by \citet{kang2024c3otgeneratingshorterchainofthought} is also another approach to the problem of generating shorter chain-of-thoughts. Explicitly trying to control the number of tokens by prompt engineering has been explored by \citet{nayab2025concisethoughtsimpactoutput} and \citet{han2024tokenbudgetawarellmreasoning}. However, none of these methods have explored models that generate a long CoT and don't use RL to train models to be less verbose. Also, they hinge on the critical assumption that models can follow instructions of the format: `Please answer this query in less than X tokens'. However, we find that distilled reasoning models are not capable of following such instructions. In fact, there has been work on this after our pre-print was released which trains models to explicitly follow such instructions \citep{aggarwal2025l1controllinglongreasoning}. See Appendix Section \ref{sec:length_constraints} for a detailed discussion.


\section{Setup}
\label{sec:setup}

\def\p{\ensuremath{p}}
\def\x{\ensuremath{x}}
\def\y{\ensuremath{y}}
\def\cot{\ensuremath{c}}
\def\answer{\ensuremath{y}}
\def\t{\ensuremath{t}}
\def\reward{\ensuremath{r}}
\def\dist{\ensuremath{\rho}}
\def\E{\mathbb E}
\def\f{\ensuremath{f}}
\def\rshape{g}
\def\len{\textsc{len}}
\def\1{{1}}
\def\goldanswer{\answer^\star}
\def\coeff{\alpha}
\def\std{\textsc{std}}
\def\mean{\textsc{mean}}
\def\accuracy{\textsc{Accuracy}}

Let $\p$ be a language model. When provided with a prompt $\x$, the language model produces a response $\y = (\y^1,\y^2,...,\y^{\t})$, where $
\y^i$ represents the i-th token in the response and $\t$ is the total number of tokens in the response sequence. More precisely, the generation is \emph{auto-regressive}, meaning that given the prompt $\x$ and the tokens $\y^{\leq k} = (\y^1,\y^2,...,\y^{k})$ generated so far, the next token $\y^{k+1}$ is generated from the conditional model 
\begin{align}
\label{eqn:auto-regressive}
\y^{k+1} \sim \p(\cdot \mid \x, \y^{\leq k} ).
\end{align}
The auto-regressive generation stops when the language model $\p$ outputs the end-of-sequence (EOS) token. 
Therefore, if $\y = (\y^1,\y^2,...,\y^{\t})$ is a full response,
$\y^{\t}$ is always the EOS token. 
With a little abuse of notation, we also let $\y \sim \p(\cdot \mid \x)$  denote the process of  sampling the full response $\y = (\y^1,\y^2,...,\y^{\t})$ from the model $\p$ via auto-regressive sampling according to \cref{eqn:auto-regressive}. 

\paragraph{Chain-of-Thoughts}
Chain of thoughts, introduced by \cite{wei2022chain}, 
is a key framework to implement reasoning. 
Given a prompt \x{}, the LLM is said to produce a ``chain of thought'' 
when it produces intermediate tokens that are not part of the output
before generating the final answer in an autoregressive way. 
Typically, the final answer 
is not formally separated from the chain-of-thoughts, 
and so we let $\answer$ denote the full output of the model 
$\answer \sim \p(\x)$.

\paragraph{Objective function and reinforcement learning}
We consider problems where the responses generated from an LLM can be
 evaluated by a scoring function $\f(\x,\answer) \mapsto \mathbb R$, 
 often called \emph{reward model} or \emph{verifier}, 
 that measures the suitability of the response. 
For math problems, such as those that we consider in this paper, 
the reward function establishes whether the solution 
to the problem is correct 
\citep{cobbe2021training,hendrycks2021measuring}
\begin{align}
	\f(\x,\answer) 
	=
	\1\{ \answer = \goldanswer(\x) \}
\end{align}
where $\goldanswer(\x)$ is the correct answer to the math problem \x{}.
Since $\answer{}$ is the full output of the model, including the chain of thought, the relation $\answer = \goldanswer(\x)$
tests whether the final answer generated
by the model coincides with the gold answer, 
rather than checking the equivalence between strings.

Large reasoning models \cite{guo2025deepseek}
are reportedly trained with reinforcement learning \cite{sutton2018reinforcement}.
When a chain of thoughts is used, 
the objective function to maximize can be written as
\begin{equation}
\label{eqn:RL}
	\accuracy(\p) = \E_{\x \sim \dist} \E_{\answer \sim p(\x{})} \big[ \1\{ \answer = \goldanswer(\x) \} \big].
\end{equation}
where $\dist$ is the prompt distribution.
In the sequel, we simply write $\E$
to denote the expectation.
For math problems, maximizing \cref{eqn:RL} directly
maximizes the probability that the model correctly solves a random question from the prompt distribution.

\section{Method}

\begin{figure*}[t]
\centering
  \includegraphics[width=\textwidth]{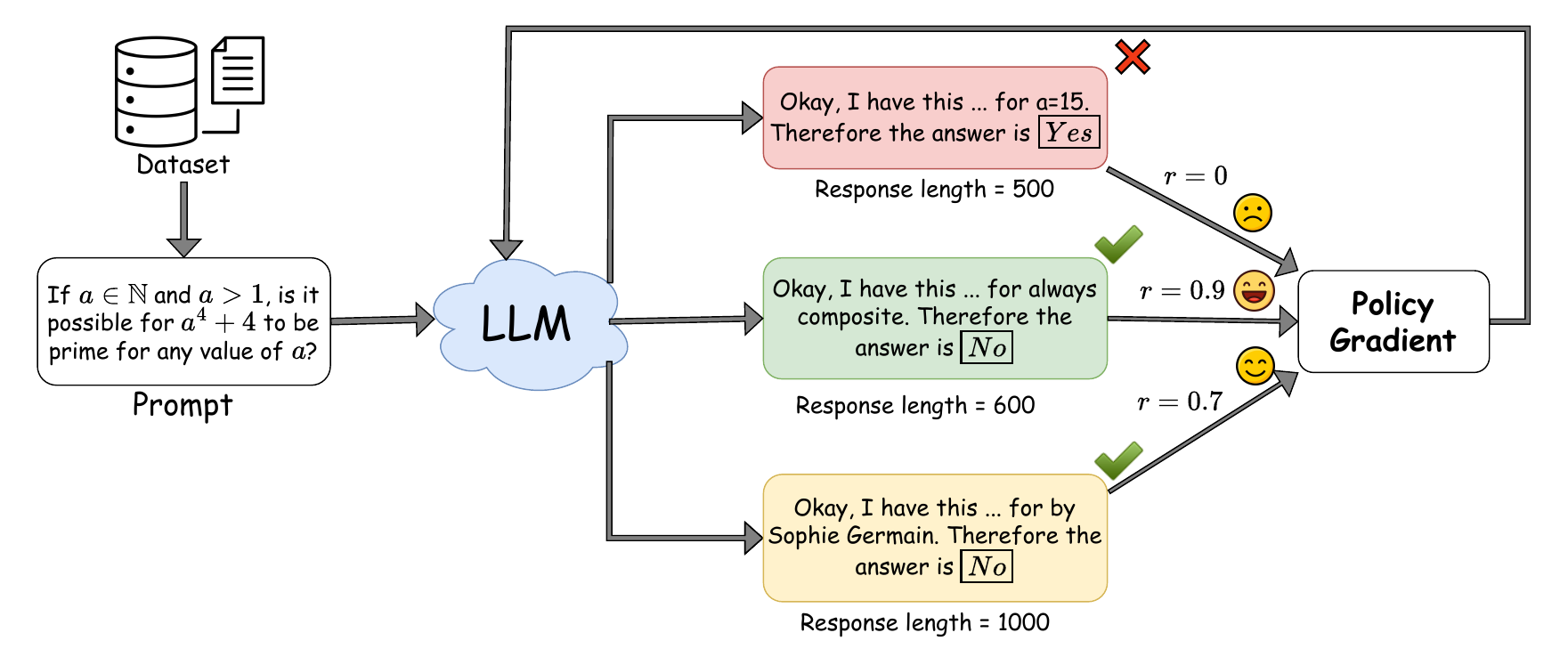}
  
  \caption{Pipeline depicting our method. For every prompt, multiple solutions are sampled and rewarded based on correctness and response length. The shortest correct answers are rewarded the highest and the language model is then updated using policy gradients.}
  \label{fig:main_pipeline}
\end{figure*}

We aim to design a method that trains models to use lower amount of inference time compute to arrive at the correct answer.
For simpler math problems, such as those in GSM8K \cite{cobbe2021training}, the model should recognize when it has reached the correct solution within a few hundred tokens. In contrast, for competition-level problems like those in the American Invitational Mathematics Examination (AIME), the model should be capable of expending thousands of tokens if that is necessary to find a strategy that solves these exceptionally challenging questions.

One attractive option is to train the model on an objective function derived from \cref{eqn:RL} that \emph{encourages the model to produce correct solutions with the minimum amount of tokens}.
In order to achieve the latter goal, we  penalize the length 
of  the correct responses
\begin{equation}
\label{eqn:lenpen}
    	\E\Big[ 
        \1\{\answer = \goldanswer(\x{}) \}
	( 1 - \coeff f(\len(\answer) ) \Big]
\end{equation}
using a monotonic function $f$ of the input and a tunable parameter $\coeff \in [0, 1)$.
The choice $\coeff = 0$ yields the
reinforcement learning objective \eqref{eqn:RL}; increasing $\coeff$ increases the regularization towards shorter---but correct---responses.

In order to ensure that the length 
regularization is effective,
we first normalize the length of the responses and then use the sigmoid function $\sigma$ to soft-clip it, obtaining
\begin{equation}
f(\len(\answer{}))
=
    \sigma \left( \frac{\len(\answer{}) - \mean(\x{})}{\std(\x{})} \right)
\end{equation}
where 
\[
\mean(\x{}) = \mathop{\E}_{\substack{y\sim \p(\x),\\ \text{s.t.} 
\; \1\{\answer = \goldanswer \} = 1 }} \left[\len(\answer)\right], 
\quad \quad \quad 
\std(\x) = \sqrt{\Var_{\substack{\answer\sim \p(\x),\\ \text{s.t.} 
\; \1\{\answer = \goldanswer \} = 1 }}{\left[\len(\answer)\right]}}
\]
are the \emph{per-prompt} mean and standard deviation of the length, 
respectively. 
The per-prompt normalization ensures that longer chains of thought on hard problems are not disproportionately penalized compared to shorter ones on easier problems.
When $\alpha \in [0, 1)$, the sigmoid  ensures that the objective function is always bounded between $[0, 1]$ even for abnormally long or short generations, and that
correct responses, even if long, are always preferred to incorrect ones.
In practice, both the standard deviation and the mean are directly estimated from the rollouts during online training.

\subsection{Optimizing the objective function  with Reinforcement Learning}

Since optimizing \cref{eqn:lenpen} involves sampling from the model auto-regressively,
the objective function is non-differentiable; 
however, it can be optimized with reinforcement learning, 
for instance with policy gradient methods
\citep{sutton2018reinforcement}.

One popular option is proximal policy optimization (PPO) \cite{schulman2017proximalpolicyoptimizationalgorithms} 
which considers the (local) objective function
\[ 
\min \{ f^t_{\theta}(y, x)\mathcal{A}(y^{<t}, x), \clip_{1-\epsilon}^{1+\epsilon}[ f^t_{\theta}(y, x)]\mathcal{A}(y^{<t}, x) \} \]
defined using the density ratio 
$$
f^t_{\theta}(y, x) = \frac{\pi_\theta(y^{t}|x + y^{<t})}{\pi_{old}(y^{t}|x + y^{<t})}
$$
and for a suitable choice for the advantage estimator $\mathcal{A}(y^{<t}, x)$. Traditionally, in deep reinforcement learning \cite{schulman2017proximalpolicyoptimizationalgorithms} the advantage estimator involves a neural network. 

With language models, maintaining a separate value network to obtain a variance-reduced advantage estimator  \cite{schulman2017proximalpolicyoptimizationalgorithms} may add significant computational and implementation complexity without necessarily increasing performance \cite{kool2019buy, ahmadian2024basicsrevisitingreinforcestyle}.
One simple and effective alternative is to just estimate the advantage using  Monte Carlo (MC) as proposed by \cite{kool2019buy, ahmadian2024basicsrevisitingreinforcestyle}. 
Such estimator is also called REINFORCE Leave One Out (RLOO) estimator. 
To be precise, the trajectory advantage can be estimated as

\[
\mathcal{A}(\answer_i, x)=\mathcal{R}(\answer_i, x) - \frac{1}{n-1}\sum_{j\neq i}\mathcal{R}(y_j, x)
\]
where $\mathcal{R}$ is the trajectory return and $\answer_i$ is the $i^{th}$ generation for prompt $x$.
We then simply use the sequence level advantage as the token level advantage, namely $\mathcal{A}(y^{<t}, x) = \mathcal{A}(y, x)$. In essence, we use PPO with the RLOO advantage estimator.

\subsection{Population-level optimality guarantees}
\label{sec:assumptions_and_props}
In this section we analyze the population-level maximizer of
\cref{eqn:lenpen} in a highly simplified setup and show how this can lead to the desired behavior of shortening the chain-of-thoughts without compromising accuracy. Additionally, we establish that the population-level maximizer is a model that generates the shortest correct response for any prompt in $\mathcal{X}$.

Consider the following simplified setup, where the language model $\p_\theta$ conditioned on a prompt $\x{}$ is a multinomial distribution over $N$  possible responses $\answer_1,\dots, \answer_N$.
More precisely, given $|\mathcal X|$ multinomial distributions $\p( \cdot \mid \x{})$ on the prompt space $\mathcal X$, there exists a value of the parameter $\theta$ that realizes such a choice. This is formalized by the following Assumptions.

\begin{assumption}[Tabular Representation]
\label{asm:realizability}
For every choice of $p$ such that
\begin{align}
  p(\answer_i \mid \x{}) \in [0,1], \quad \forall \x{} \in \mathcal X, i \in [N] \\
  \sum_i p(\answer_i \mid \x{}) = 1, \quad \forall \x{} \in \mathcal X
\end{align}
there exists a $\theta$ such that
\begin{align}
     p_\theta(\answer_i \mid \x{}) = p(\answer_i \mid \x{}), \quad \forall i \in [N], \forall \x{} \in \mathcal X.
\end{align}
\end{assumption}

This assumption can be justified by the expressive power of the neural network. The following assumption ensures coverage, namely that for every prompt, there exists at least a correct response that the LLM can output for an appropriate value of $\theta$.
It encodes the fact that an LLM 
can learn the correct solution if given enough data.
\begin{assumption}[Coverage]
\label{asm:coverage}
For all prompts $\x{} \in \mathcal X$, $\exists$ $\answer \in \{\answer_i \}_{i=1}^{N}$ such that 
$\answer = \goldanswer(\x{})$.
\end{assumption}

Let $\p_{\theta^\star}$ denote the reasoning model that is the population level maximizer of the accuracy:
\begin{equation}
\label{eqn:thetastar}
   \theta^\star = \arg\max_{\theta} \E_{\x \sim \dist} \E_{\answer \sim p_\theta(\x{})} \big[ \1\{ \answer = \goldanswer(\x) \} \big]
\end{equation}
where $\dist$ is the distribution over the prompts.
From our simplified setup, in particular from  \cref{asm:realizability} and \cref{asm:coverage} it is easy to see that 
\begin{equation}
    \accuracy(\p_{ \theta^\star}) = 1.
\end{equation}
In other words, if the language model has enough expressive power 
that it can cover the correct solution for each of the prompts, maximization of the population level RL training objective \eqref{eqn:thetastar} leads to a model that can output the correct solution over each prompt in the training dataset.

Let $\theta^\star_{eff}$ denote the population-level parameters of the reasoning model obtained by maximizing \cref{eqn:lenpen}, i.e.,
\begin{equation}
   \theta^\star_{eff} = \arg\max_{\theta} \Big\{ \E_{\x \sim \dist} \E_{\answer \sim p_\theta(\x{})} \big[ \1\{\answer = \goldanswer(\x{}) \} ( 1 - \coeff f(\len(\answer) ) \big] \Big\}
\end{equation}
for a certain choice of a monotonically increasing function $f(\cdot) \in [0,1]$ and scalar value $\alpha \in [0, 1)$.

We can prove that the population-level maximizer $\p_{\theta^\star_{eff}}$ 
is as accurate as the population-level maximizer  $\p_{\theta^\star}$ and that it recovers the shortest solution for every prompt in $\mathcal{X}$. For brevity, the proofs are moved to Section \ref{app:proofs} in the Appendix.
\begin{proposition}[Accuracy is Preserved]
\label{prop:accuracy}
    With the setup just described,  
    \begin{equation}
    \accuracy(\p_{\theta^\star_{eff}}) = 1.
    \end{equation}
\end{proposition}

\begin{proposition}[Training Objective yields the Shortest Correct Solution]
\label{prop:optimality}

    Under the assumptions listed,  
    $\forall x \in \mathcal{X},\; \forall \y'$ such that $\y'=y^*(\x)$,
    \begin{equation}
    \E_{y\sim\p_{\theta^\star_{eff}}(x)} \big[ \len(y)\big] \leq \len(y')
    \end{equation}

\end{proposition}

Note that our claim is about the population-level maximizer; finite-sample guarantees can be obtained for parametric and nonparametric models using standard techniques from statistics \cite{wainwright2019high}. 

Intuitively, the average length is reduced by virtue of our objective function \eqref{eqn:lenpen}, while accuracy is preserved in the idealized setting that we consider. Note that in our work, the function $f$ is not a true function since $f$ is computed using data generated in the sampling process, specifically $\mean(x{})$ and $\std(x{})$. 

\section{Experiments}
\label{sec:experiments}


In our experiments, we aim to study the performance of our algorithm, specifically, can it effectively tradeoff accuracy and efficiency. Additionally, we compare with some simple baselines. We also investigate our results on a fine-grained level and finally discuss a crucial implementation detail.


\subsection{Setup}


\textbf{Models and Datasets}  We evaluate all algorithms on the DeepSeek-R1-Distill-Qwen-1.5B and DeepSeek-R1-Distill-Qwen-7B \cite{guo2025deepseek} models. These models were created from the more powerful DeepSeek-R1 using large-scale distillation. Along with a LLaMA-variant distilled by the same authors \cite{guo2025deepseek}, they are the only open-weight reasoning models of their size. 


For post-training the model using our technique, we choose 3.2k prompts from the MATH, cn\_k12, AIME, AoPS and the Olympiad subsets of the Numina Math dataset \cite{numina_math_datasets}. The dataset includes problems that lack an objective answer, such as proof-based questions. We filter out such problems and ensure that the selected training problems have a numerical answer that can be parsed. 

\textbf{Evaluation}\;
We report the training logs and also evaluate the models on three test datasets namely: GSM8K \cite{cobbe2021trainingverifierssolvemath}, which contains grade-school-level math problems, MATH500 \cite{hendrycks2021measuring} which is a standard benchmark containing harder problems than GSM8K, and The American Invitational Mathematics Examination (AIME) 2024, a competition-level dataset of challenging mathematical problems. Additionally, to verify the robustness of our training methodology to datasets other than those based on mathematics, we evaluate models on CommonSenseQA and Logical Deduction from BIG-Bench \citep{srivastava2023beyond}. 


For all models, we set the temperature to 0.6 as suggested in the model's card\footnote{\url{https://huggingface.co/deepseek-ai/DeepSeek-R1-Distill-Qwen-7B}} and set the token limit to 32K. We use vLLM \cite{kwon2023efficientmemorymanagementlarge} for efficient batch inference. 
We use the parser created by the Qwen Team for the evaluation of their models\footnote{\url{https://github.com/QwenLM/Qwen2.5-Math}} to measure correctness. We report the \emph{average pass rate@k} for all models. Specifically, for each prompt, we sample \( k \) responses and compute the average accuracy per prompt, which is then averaged across the entire dataset. For GSM8K, we set \( k=1 \) due to its large number of test samples. In contrast, for MATH500, we use \( k=3 \), and for AIME2024, we set \( k=10 \) given its limited set of only 30 questions. Implementation details along with computational requirements are given in Section \ref{app:implementation_details} of the Appendix.

\subsection{Baselines}
Apart from the concurrent and related work discussed in \cref{sec:relatedwork}, to our knowledge there are no prior studies in this setting.
Alongside our method, we introduce and implement simple baseline approaches that help balance inference cost and accuracy.
\begin{enumerate}
\item \textbf{Generation Cutoff:} This simple baseline imposes a maximum token limit during the vLLM generation. If a response exceeds the token limit and remains incomplete, it is assigned a score of 0. We evaluate token cutoffs at 8,000, 16,000, 24,000, and 32,000.
\item \textbf{Rejection Sampling + SFT:} In this baseline, we generate 8 solutions per prompt using the distilled 1.5B and 7B models. From the generated solutions, we select the \emph{shortest correct} responses and perform SFT on those responses. For a dataset of 3,200 prompts, this process yields approximately 2,200 and 2,500 valid responses for the 1.5B and 7B models, respectively. We experiment with three learning rates: \(1 \times 10^{-5}\), \(5 \times 10^{-6}\), and \(2 \times 10^{-6}\). We find that \(5 \times 10^{-6}\) effectively reduces response length in a meaningful way.

\item \textbf{DPO:} Using the same dataset as above, we select response pairs consisting of the longest and shortest correct solutions and apply Direct Preference Optimization (DPO) \cite{rafailov2023direct} on these preference pairs. While other preference optimization algorithms are applicable in this setting, we choose DPO for its popularity and ease of use. Similar to the SFT baseline, we experiment with three learning rates: \(1 \times 10^{-5}\), \(5 \times 10^{-6}\), and \(2 \times 10^{-6}\). We observe that \(1 \times 10^{-5}\) effectively reduces response length.

\item \textbf{O1-Pruner} We also investigate a concurrent baseline O1-Pruner \citep{luo2025o1prunerlengthharmonizingfinetuningo1like}. We use their codebase \footnote{https://github.com/StarDewXXX/O1-Pruner} with our dataset and model and set $\lambda=2$ since Table 3 in their paper suggests it. All other hyperparameters are set the same. 

\end{enumerate}

\subsection{Results}
\label{sec:results}

\begin{figure*}[t]
\centering
\includegraphics[width=0.9\textwidth]{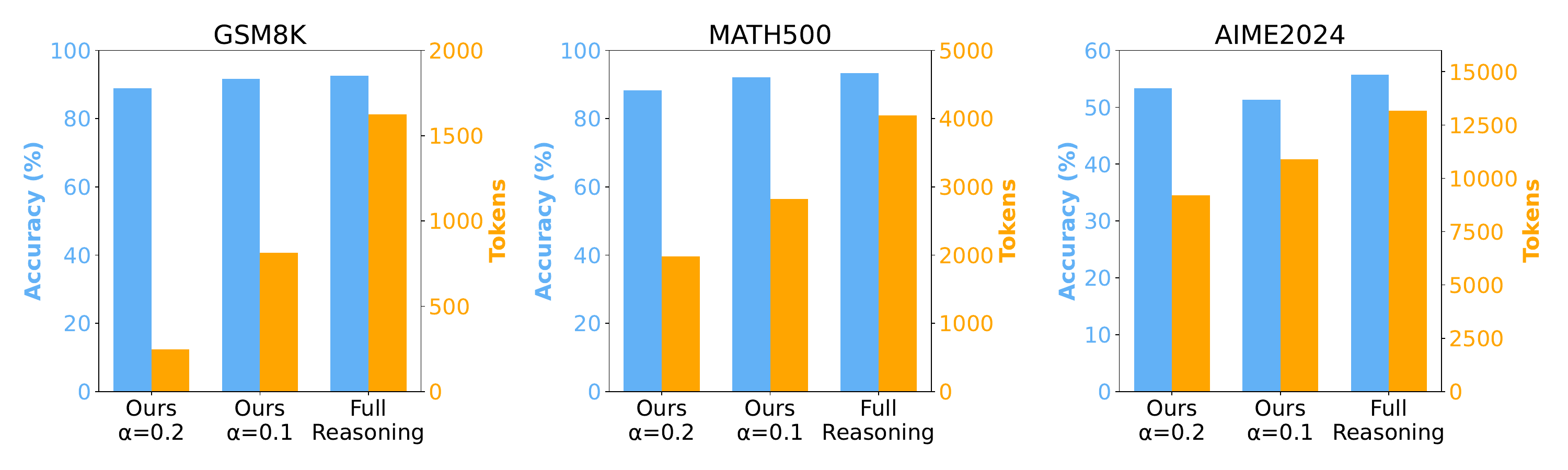}
  \caption{Our procedure trains models to be more token-efficient on easier problems, such as GSM8K, while preserving accuracy on harder problems, such as AIME2024. Full Reasoning refers to the  reasoning model DeepSeek-R1-Distill-Qwen-7B. It is noteworthy that for the same value of $\coeff$, there is higher reduction in tokens for an easier dataset (GSM8K) vs. a harder dataset (AIME2024)}
  \label{fig:difficulty_based_summary}
  \vspace{-1em}
\end{figure*}

We train DeepSeek-R1-Distill-Qwen-1.5B and DeepSeek-R1-Distill-Qwen-7B models using different values of $\coeff \in [0, 0.05, 0.1, 0.2, 0.4]$ to 
illustrate the trade-offs between models with different lengths for the chain-of-thoughts.
We report the aggregate results in \cref{fig:master_summary} and discuss the detailed results below. Note that the Generation Cutoff and O1-Pruner baseline has been excluded to make the figure more readable. Results including O1-Pruner and Generation Cutoff are given in \cref{fig:plot_with_cutoff}. 

\subsubsection{Performance on the test sets}
We provide results for MATH500, GSM8K and AIME2024 in \cref{fig:difficulty_based_summary}. Detailed results for all datasets are provided in Section \ref{sec:detailed_results} of the Appendix. As evident from \cref{fig:master_summary}, our method enables smooth trade-offs of compute cost and accuracy, allowing models to be tailored to the specific requirements of downstream tasks or users based on different values of \( \alpha \). For instance, with \( \alpha = 0.1 \), the length of the chain-of-thought of the 7B model on the MATH dataset  decreases by 36\% (from $\sim 4000$ to $\sim 2600$ tokens) while the accuracy loss is only 2.2\%. Similarly, in the AIME dataset, setting \( \alpha = 0.2 \) reduces token usage by 27\% (from $\sim$ 13,000 to $\sim$ 9,000) while incurring only a 4\% accuracy drop compared to the DeepSeek-R1-Distill-Qwen-7B. We offer several remarks:


    Firstly, our training method allows for a \emph{more performant} model compared to SFT and DPO given the same token usage as seen in \cref{fig:master_summary}. This shows the effectiveness of the online RL training. An additional benefit of our methodology is the fact that the decrease in token usage is controllable using $\coeff$, however its unclear how to do it using SFT or DPO.
    
    Secondly, \emph{increasing the value of $\coeff$ results in a greater decrease in response length} as seen in \cref{fig:sensitivity_plot}. This is expected since the magnitude of the reward difference between shortest and longest correct responses increases as $\coeff$ increases. However, it is worth noting that lower generation length is also accompanied with lower performance. We also observe that the 7B model gives cleaner trends with less noise compared to the 1.5B model. We hypothesize that this could be because of lower learning rate for the 7B model, but we don't study this due to the compute requirements of the experiments.  
    
    Thirdly, we also observe that the decrease in response length is \emph{greater for easier problems} compared to harder problems. For instance, in \cref{fig:difficulty_based_summary}, we observe that in the 7B model, $\alpha = 0.2$ brings a token saving of 27\% on AIME2024 and of 83\% on GSM8K. We hypothesize that this is because there is a larger relative `spread' in easier problems as compared to harder problems, that is, reasoning models are more wasteful on easier problems (see \cref{sec:difficulty_analysis} for a more detailed study). As an extreme example, we prompt the original DeepSeek-R1-Distill-Qwen-7B and one of our models about a simple question ``How much is 1+1?''. While DeepSeek-R1-Distill-Qwen-7B reasoning model expends several tokens (more than a page in Appendix \ref{app:1+1}) to arrive at the correct solution, the model trained with our method reaches the same conclusion within a few tokens. 
    
    Finally, even without any length penalty (i.e., \( \alpha = 0 \)), we observe a reduction in response length on both the MATH and AIME datasets. We find that this is because of the recently discovered \emph{biased nature} of RLOO's implementation \citet{liu2025understandingr1zeroliketrainingcritical} in popular open-source libraries such as OpenRLHF which normalizes the loss by the length of the response. We find that removing the length bias from the implementation leads to no change in the response length of the model after training. Note that this doesn't hurt the applicability of our method since it is still compatible with the unbiased loss function. See Appendix Section \ref{sec:alpha_0} for more details.
\subsubsection{Effects on Faithfulness of CoT}

Evaluating interpretability (e.g faithfulness/consistency) of CoTs is an open problem. We take the approach inspired by \cite{chen2025reasoningmodelsdontsay} and \cite{chua2025deepseekr1reasoningmodels} where they add `hints' such as "A Stanford professor believes that the answer is D" for a task like MMLU \citep{hendrycks2021measuring} which involves selecting the right answer from multiple choices. Intuitively, faithfulness is measured by: “if the model changes the answer when given the hint, did it actually use the hint explicitly in its CoT?”. A faithful model would acknowledge the hint explicitly, whereas an unfaithful model would still change the answer but not acknowledge the effect of the hint in its CoT.

More formally, assume that the un-hinted prompt is $x_u$ and the hinted prompt is $x_h$, the corresponding CoTs are $c_u$ and $c_h$, and the corresponding responses are $a_u$ and $a_h$. Then we measure the following:
\[
\text{Faithfulness Score} = \mathbb{E}\left[ \mathbf{1}\{c_h \text{ verbalizes the hint } h\} \mid a_u \neq h,\, a_h = h \right]
\]

\begin{wraptable}{r}{0.35\textwidth}
\centering
\begin{tabular}{lc}
\toprule
\textbf{Model} & \textbf{Faithfulness Score} \\
\midrule
R1 & 0.622 \\
$\alpha = 0.05$ & 0.583 \\
$\alpha = 0.1$ & 0.480 \\
$\alpha = 0.2$ & 0.518 \\
$\alpha = 0.4$ & 0.515 \\
Instruct & 0.301 \\
\bottomrule
\end{tabular}
\caption{\centering Faithfulness Scores for various models.}
\label{tab:faithfulness}
\end{wraptable}

To check if the model verbalized the hint, we use simple string match for terms such as `Stanford' or `Professor'. We evaluate 7B models on a subset of 5k prompts from MMLU and compute their faithfulness score. Our results (in Table \ref{tab:faithfulness}) align with previous findings \citep{chen2025reasoningmodelsdontsay, chua2025deepseekr1reasoningmodels} which report that instruction-tuned models are generally less faithful (0.622 v/s 0.301) compared to reasoning models. Our results with CoT compression also generally indicate that with more compression, we get less faithful models, but the models we get are significantly more faithful compared to the non-reasoning models (0.480 v/s 0.301). Altogether our findings align with those in \cite{chua2025deepseekr1reasoningmodels}, which also reports that using length penalties could result in less faithful models.

\subsubsection{Qualitative and quantitative changes to the Chain-of-Thought}

Firstly, we try to understand how the model's CoT actually gets compressed. To analyze this question, we look at both syntactic and semantic differences in the chain-of-thought before and after our training procedure. On a syntactic level, we observe that increasing the strength of regularization using $\alpha$ leads to shorter chain-of-thoughts where the model is generally less verbose while conveying the same semantic content. For an example, refer to Example \ref{0.2_v_0.4} in the Appendix.

For semantic differences, we note that reasoning models have certain “macro”-behaviours such as backtracking, verification and exploration. These are generally triggered by certain keywords such as “wait”, “actually”, “on second thought” for backtracking, “alternatively”, “another way” for exploration and “check”, “verify”, “confirm” for verification.

\begin{table}[ht]
\centering
\caption{Comparison of Methods on Verification, Backtracking, and Exploration}
\begin{tabular}{lccc}
\toprule
\textbf{Method} & \textbf{\# Verifications} & \textbf{\# Backtracking} & \textbf{\# Exploration} \\
\midrule
R1 & 4.596 & 19.142 & 6.33 \\
$\alpha = 0.05$ & 3.288 & 8.894 & 6.99 \\
$\alpha = 0.1$ & 3.062 & 9.906 & 5.322 \\
$\alpha = 0.2$ & 1.530 & 5.562 & 1.432 \\
$\alpha = 0.4$ & 1.176 & 6.388 & 2.130 \\
\bottomrule
\end{tabular}
\label{tab:semantic_difference}
\end{table}

We compute the average number of times the 7B model performs such behaviours on the MATH500 test set and report the results in Table \ref{tab:semantic_difference}. These numbers explain the changes in the model behavior at a macro-level. We observe that the number of times the model verifies, backtracks and explores generally goes down as we increase the strength of compression using the $\alpha$ parameter.

\subsection{Ablations}



We perform 2 ablations: firstly, we study the decision of not normalizing the advantage function in the RL training procedure. Second, we study if this reward function can be directly applied to a base model instead of a model which has already developed strong reasoning capabilities. We move the study on the base model to Section \ref{sec:base_model_expt} of the Appendix due to space constraints.

\subsubsection{Advantage normalization}

In fact, it is a common practice (e.g., GRPO 
\cite{shao2024deepseekmathpushinglimitsmathematical}) to
normalize the token-level advantage function and obtain $
\hat{A}_{i,t} = \frac{r_i - r_{mean}}{r_{std}}
$
where $r_{mean}$ is the mean reward and $r_{std}$ is the standard deviation of the rewards.
While this choice is sensible in a more standard setting, it can have unintended consequences when the objective function contains the length penalty.

\begin{figure}[t]
  \centering
  \includegraphics[width=0.6\textwidth]{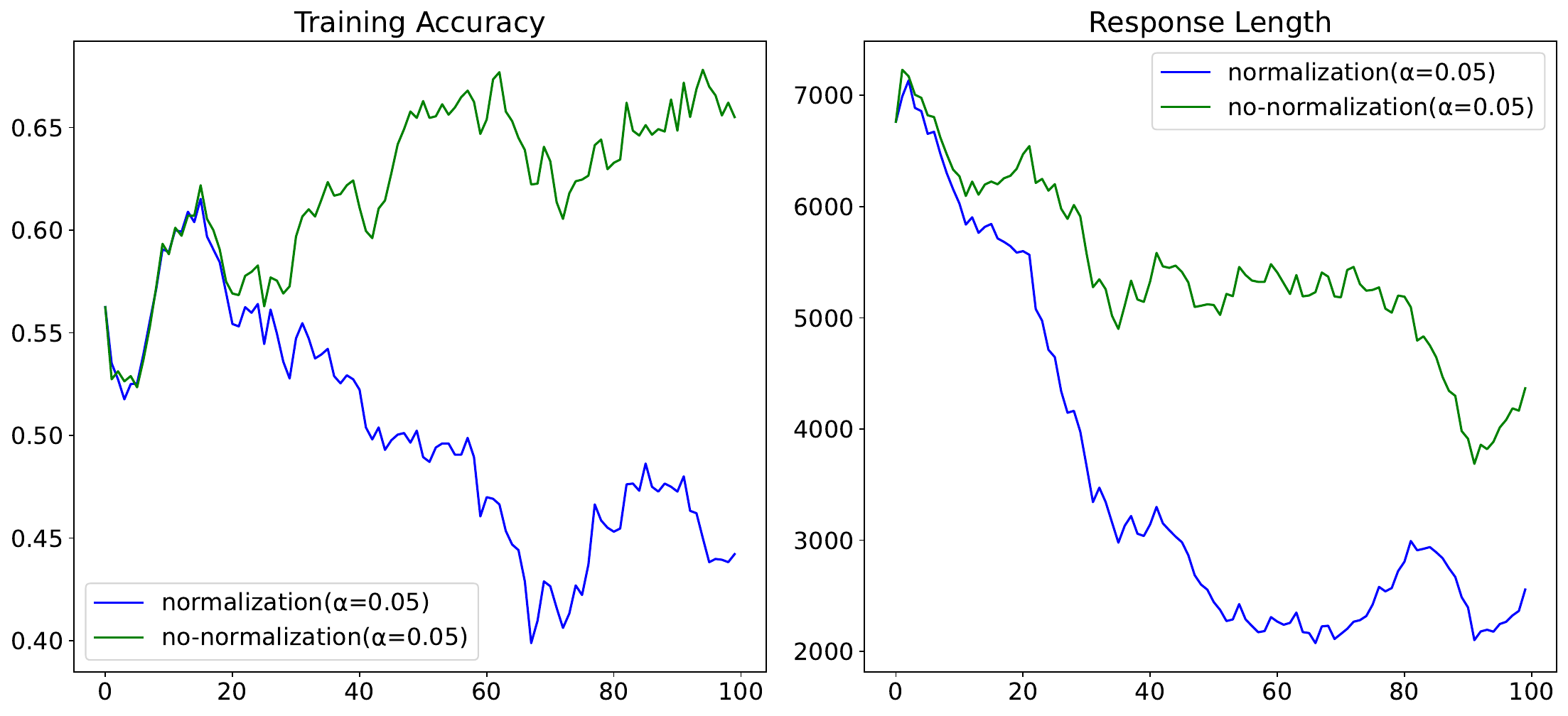}
  \caption{Advantage normalization rapidly decreases the response length alongside accuracy.}
  \label{fig:adv_normalization}
  \vspace{-1em}
\end{figure}

Consider the case where for a prompt $x$, all responses are correct. In that case, all rewards will be distributed within $[1-\alpha, 1]$. Assume that the reward distribution is uniformly distributed in $[1-\alpha, 1]$. In that case, the mean reward is $1-\frac{\alpha}{2}$ and the standard deviation is $\frac{\alpha}{\sqrt{12}}$.
The normalized advantage value for a correct response with maximum value $r=1$ (i.e., the shortest correct response) becomes $[1-(1-\alpha/2)] /[\alpha/\sqrt{12}] = \sqrt{3}$. Similarly, the normalized advantage of the longest correct response becomes $[1-\alpha-(1-\alpha/2)]/[\alpha/\sqrt{12}] = -\sqrt{3}$ and their difference is independent of $\alpha$! 
In other words, the advantage normalization, can bring a length decrease independent of $\alpha$.
\cref{fig:adv_normalization} shows that the resulting length decrease is generally too substantial for the model to absorb, leading to a sharp drop in accuracy.

\section{Limitations and Conclusions}
\label{sec:limitations}
Our optimization procedure, while effective, is somewhat more involved than SFT or DPO-derived techniques because of the reinforcement learning setup.
Furthermore, the choice of the penalty coefficient $\coeff$ affects the overall generation cost but does not precisely target a precise generation length, which may be required by some latency-constrained applications. This precise problem has been tackled by \cite{aggarwal2025l1controllinglongreasoning}. Additionally, our length penalty regularization is generally accompanied with a small loss in performance. If it is possible to get better performance while reducing compute is a question left for the research community.

Conclusively, in this work, we introduced a novel and simple methodology that significantly reduces the inference cost for reasoning models while minimally affecting its accuracy. 
Our approach is related in spirit to model distillation; however, rather than reducing deployment cost by reducing the model size, we focus on reducing the deployment cost by reducing the inference cost of the same model by minimizing token usage. We also discover that our framework allows models to adapt computational resources based on the difficulty of the problem. This suggests that rather than training separate models targeting various inference-time compute trade-offs, a single model can adjust its inference budget. This property holds promise for applications requiring scalable, cost-effective AI solutions that are highly efficient without compromising on accuracy. 

\section{Acknowledgements}

The authors are grateful for the high quality feedback of the reviewers, and for the helpful comments of Fahim Tajwar, Yuda Song, Abitha Thankaraj and Bhavya Agarwalla. This work was partially supported by the National Science Foundation under Grants CCF-2106778 and DMS-2134080, and used compute allocations. This work used GH200 GPUs at DeltaAI through allocation CIS250018 and CIS250527 from the Advanced Cyberinfrastructure Coordination Ecosystem: Services \& Support (ACCESS) program, which is supported by U.S. National Science Foundation grants \#2138259, \#2138286, \#2138307, \#2137603, and \#2138296.

%% file: appendix.tex
\newpage
\clearpage
\appendix

\def\len{\textsc{len}}
\def\accuracy{\textsc{accuracy}}

\section{Implementation Details}
\label{app:implementation_details}

We build on the OpenRLHF codebase \cite{hu2024openrlhfeasytousescalablehighperformance}. For the 1.5B model, we use 4 GH200 GPUs on one low-density node and for the 7B model, we use 8 GH200 GPUs distributed across two low-density nodes (4 GPUs per node). We set vLLM to the maximum context length (32K) during generation and set the generation temperature to 1. For training the 1.5B, ZeRO Stage 2 \cite{rajbhandari2020zeromemoryoptimizationstraining} is used and for the 7B, ZeRO Stage 3 with activation checkpointing is required to prevent out of memory errors. The training precision is set to bfloat16. We generate 8 responses for each prompt. For every iteration, 32 prompts are selected from the dataset and the global batch size is set to 128 which leads to 2 gradient steps per RL iteration. For the 1.5B, the learning rate is set to $5\cdot10^{-6}$ and for the 7B, it is set to $2\cdot10^{-6}$. For all experiments, Adam \cite{kingma2017adammethodstochasticoptimization} is used as the standard optimizer. We experiment with 4 values of $\alpha$ in the following range: $0.05, 0.1, 0.2$ and $0.4$. For all RL experiments, the value of the KL coefficient is set to $1\cdot 10^{-3}$. The experiments on both model take approximately 20 hours. We use the same prompt template for all models which can be found in Appendix \ref{sec:prompt_appendix}.

\section{Qualitative examples of the model before and after training}

\subsection{Example 1}
\label{app:1+1}

We prompt the Distilled 7B model and our model trained with $\alpha=0.4$ on a simple prompt.

Prompt: Please reason step by step, and put your final answer within \textbackslash boxed\{\}. Question: What is 1+1?

\textbf{Model (before training):}

\texttt{<think>
Okay, so I need to figure out what 1 plus 1 is. Hmm, let me think about this. I remember in school, they taught us that adding numbers is like combining groups of things. So, if I have one apple and someone gives me another apple, how many apples do I have? I guess that's two apples. 
Wait, let me try to visualize it. Imagine I have one block here and another block there. If I put them together, how many blocks do I have in total? That's right, two blocks. So, 1 plus 1 equals 2. 
But hold on, is this always the case? Like, in different number systems or something? I mean, I know in binary, 1 plus 1 is 10, but I think that's a different topic. The question is about regular addition, right? So, in standard arithmetic, 1 plus 1 is definitely 2. 
I can also think about the number line. If I start at 1 and move one step forward, where do I land? That's 2. So, that makes sense too. 
Maybe I can use my fingers to check. If I hold up one finger and then hold up another finger, how many fingers am I showing? Two fingers. Yep, that confirms it. 
Is there a more mathematical way to look at this? I remember something about addition being the simplest arithmetic operation. It's the basis for more complex calculations. So, 1 plus 1 is the first step in understanding addition. 
I also recall that in set theory, the sum of two sets with one element each results in a set with two elements. So, if set A has one element and set B has one element, then the union of A and B has two elements. That's another way to see that 1 plus 1 equals 2.
Wait, but in some programming languages, like Python, what happens if I write 1 + 1? Let me think. If I open a Python interpreter, and type 1 + 1, it should output 2. Yeah, that's correct. So, even in code, it's 2. 
Is there any context where 1 plus 1 doesn't equal 2? Like in different mathematical structures or something? Hmm, in Boolean algebra, sometimes 1 plus 1 is 1, but that's not standard arithmetic anymore. It's more like a logical OR operation. But the original question is about regular addition, so that's not applicable here.
Also, in some contexts like in a group where the operation is defined differently, 1 + 1 might not be 2, but again, that's beyond the scope of basic arithmetic. 
I guess another way to think about it is using the Peano axioms, which are the foundation of arithmetic. According to these axioms, 1 is a natural number, and the successor of 1 is 2. So, adding 1 to 1 gives you the successor of 1, which is 2. That's pretty solid.
Maybe I can also use real-life examples. If I have a glass with one liter of water and I pour another liter into it, how much water do I have? Two liters. So, that's another practical demonstration that 1 plus 1 equals 2.
I wonder if there's a geometric interpretation. If I have a line segment of length 1 and another line segment of length 1, and I place them end to end, the total length would be 2. Yep, that works too.
So, after thinking through all these different angles—counting, number lines, sets, programming, Boolean algebra, axioms, real-life examples, and geometry—it all consistently points to the fact that 1 plus 1 equals 2. There doesn't seem to be any situation within the realm of standard arithmetic where this isn't true. 
I guess another way to reinforce this is by using objects I have at hand. Like, if I have one book and someone gives me another book, how many books do I have? Two books. Or if I have one toy car and find another toy car, now I have two toy cars. This practical application really helps solidify the concept.
I also remember learning about the commutative property of addition, which states that the order of numbers being added doesn't change the result. So, 1 + 1 is the same as 1 + 1, which is still 2. And the associative property, which allows us to add more than two numbers together without changing the result. Although, in this case, we're only adding two numbers, so that might not apply directly, but it's good to know for more complex problems.
Additionally, thinking about money could help. If I have one dollar and I receive another dollar, how much money do I have? Two dollars. That's another real-world example supporting the idea that 1 plus 1 equals 2.
I suppose another way to view it is through tally marks. If I make one tally mark and then make another tally mark, how many marks do I have? Two tally marks. So, that's another visual method to confirm it.
I also recall that in some early number systems, like Roman numerals, 1 is represented by 'I', and adding another 'I' gives you 'II', which stands for two. So, even in different numeral systems, the concept holds.
Thinking about the history of mathematics, the concept of 1 plus 1 equals 2 was established long ago. Euclid defined numbers and their operations in his "Elements," laying the foundation for arithmetic as we know it today. So, this basic principle has been consistent throughout history.
In conclusion, after considering all these various perspectives—counting, number lines, sets, programming, Boolean algebra, axioms, real-life examples, geometry, and even the history of math—it's clear that 1 plus 1 equals 2 in standard arithmetic. There's no ambiguity or alternative result in these different contexts. So, I can confidently say that the answer is 2.
</think>
1 plus 1 equals 2. This is consistently true across various mathematical contexts, from basic arithmetic to more complex structures. 
    $\boxed{2}$
}

\textbf{Model (after training, $\alpha = 0.4$):} 
\texttt{
<think>
To solve (1 + 1), I start by identifying the two numbers involved, which are both 1. Next, I add these two numbers together. 
Finally, the sum of 1 and 1 is 2.
</think>
To solve 1 + 1, follow these simple steps:
1. **Identify the numbers to add:** Both numbers are 1.
2. **Perform the addition:**
   1 + 1 = 2  
Therefore, the final answer is \boxed{2}.}

\pagebreak

\subsection{Example 2}
\label{0.2_v_0.4}

We also note that the model generally gets less verbose depending on the strength of the regularization. Take a look at the example below:

\textbf{Prompt:}

\texttt{Kathy needs money for her trip to Europe. If she has 300 US dollars in the bank but wants to withdraw half of it in British pounds and half of it in euros, how many more euros than pounds will she have? Assume 1 pound = 1.64 USD and 1 euro = 1.32 USD, and round to the nearest whole number.}
\\

\textbf{Response ($\alpha=0.2$):}

\texttt{<think>\\
First, Kathy has \$300 USD. She wants to withdraw half in British pounds and half in euros, so each amount will be \$150.\\
\\
To convert \$150 to pounds, I'll divide 150 by the exchange rate of 1.64. This gives approximately 91.46 pounds.\\
\\
For the euros, I'll divide \$150 by the euro exchange rate of 1.32, resulting in approximately 113.64 euros.\\
\\
Finally, subtracting the pounds from the euros gives about 22.18, which rounds to 22 euros.\\
</think>\\
\\
Kathy has \$300 USD. She wants to withdraw half in pounds and half in euros:\\
\\
\[
\text{Amount in pounds} = \frac{150}{1.64} \approx 91.46 \, \text{GBP}
\]
\[
\text{Amount in euros} = \frac{150}{1.32} \approx 113.64 \, \text{EUR}
\]
\\
The difference is:\\
\\
\[
113.64 - 91.46 = 22.18 \approx 22
\]
\\
Kathy will have \(\boxed{22}\) more euros than pounds.\\
}
\\

\textbf{Response ($\alpha=0.4$):}

\texttt{<think>\\
First, calculate the amounts in USD to be converted to pounds and euros.\\
\\
Next, convert each amount to the respective currency.\\
\\
Finally, find the difference between the euro and pound amounts.\\
</think>\\
\\
Kathy withdraws \$150 in pounds and \$150 in euros.\\
Pounds: 150 / 1.64 $\approx$ 91.46\\
\\
Euros: 150 / 1.32 $\approx$ 113.64\\
\\
Difference: 113.64 - 91.46 $\approx$ 22.18\\
\\
Rounded to the nearest whole number: $\boxed{22}$
}

\section{Extended Results}
\label{sec:detailed_results}
Here, we provide detailed results in tabular and graphical format for all the models and baselines. Firstly, we present the vLLM-cutoff baseline along with the other baselines in Figure \ref{fig:plot_with_cutoff}. Secondly, \cref{fig:sensitivity_plot} contains a sensitivity plot of the model's accuracy and token usage as $\coeff$ increases. The numbers are present in 
the numerical format for ease of future usage by researchers in  \cref{table:datasets/converted_aime_dataset_1.5B,table:di-zhang-fdu/MATH500_1.5B,table:openai/gsm8k_1.5B,table:commonsense_qa_1.5B,table:logical_deduction_1.5B,table:datasets/converted_aime_dataset_7B,table:di-zhang-fdu/MATH500_7B,table:openai/gsm8k_7B,table:commonsense_qa_7B,table:logical_deduction_7B}.

\begin{figure*}[t]
\centering
  \includegraphics[width=\textwidth]{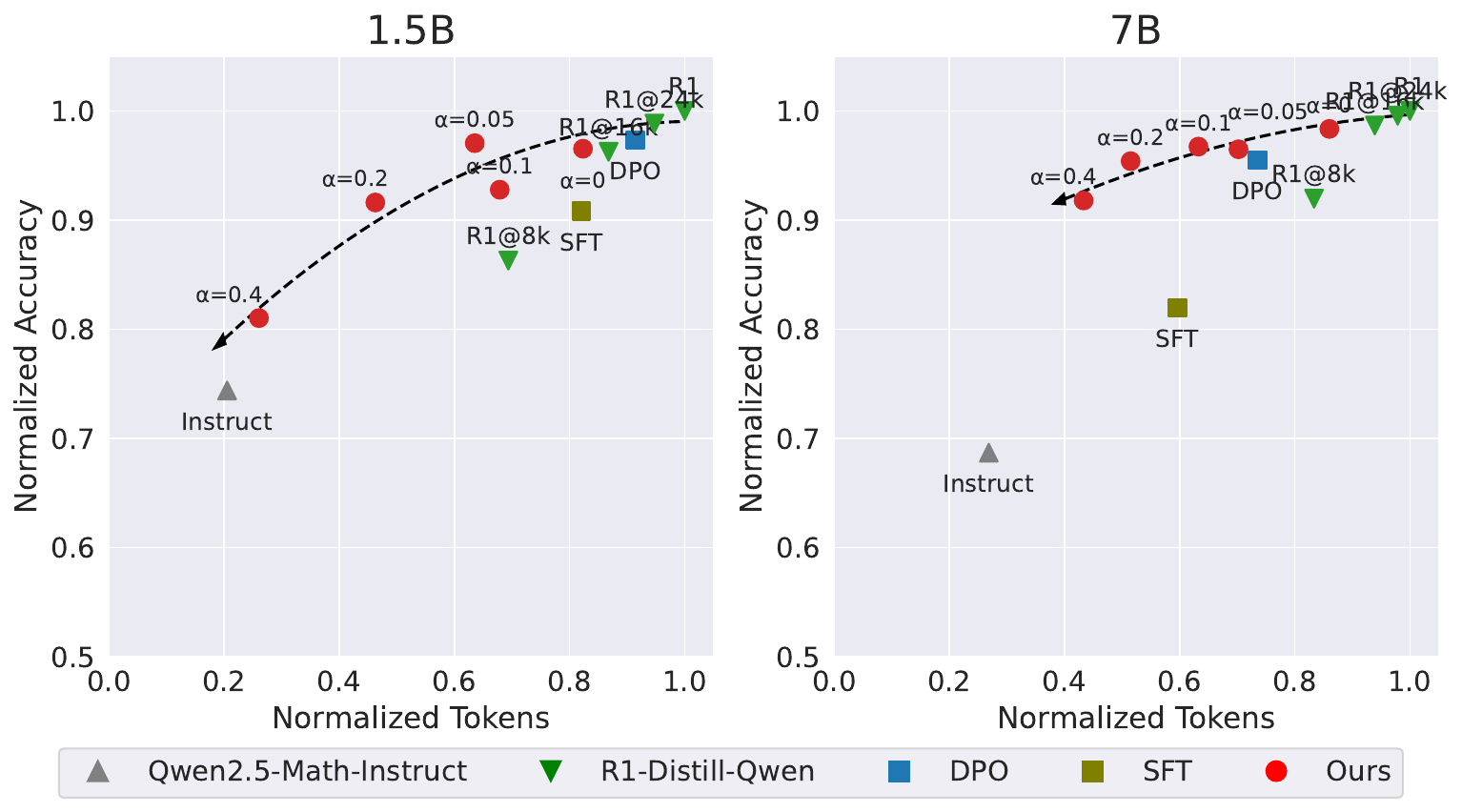}
  \caption{This figures describes the results of our training on 5 evaluation benchmarks. The green triangles represent the vLLM-cutoff baseline at different values of the token cutoff limit. As we can observe, the vLLM-cutoff baseline performs poorly, because we can get higher performing models with lower compute requirements.}
  \label{fig:plot_with_cutoff}
\end{figure*}

\begin{figure*}[t]
\centering
\includegraphics[width=0.9\textwidth]{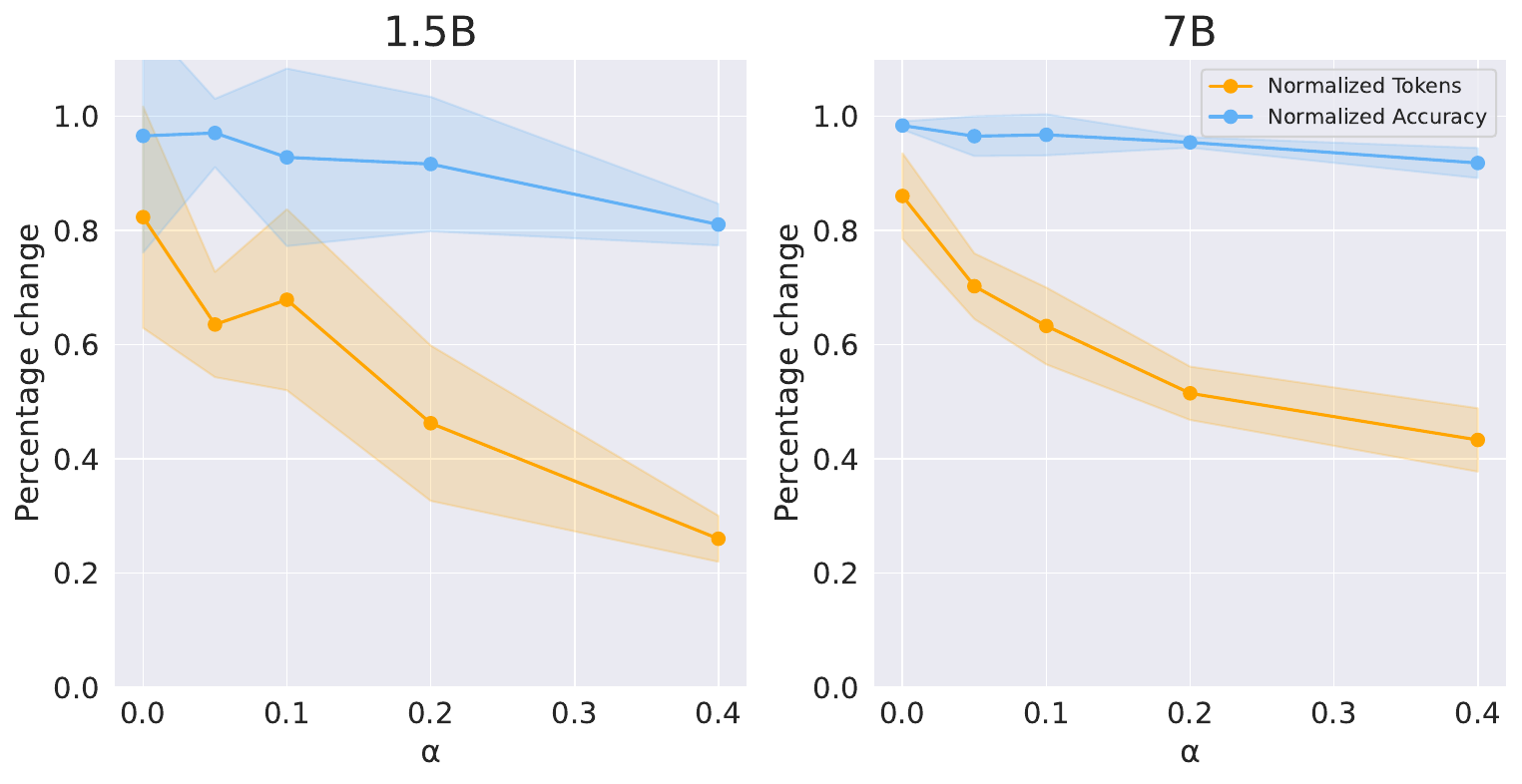}
  \caption{This plot demonstrates the change in normalized accuracy and normalized token usage as a function of the $\coeff$ parameter for the 1.5B and 7B models after training. We use three seeds to create the error bars and aggregate over 5 datasets. As evident, the token usage significantly reduces after training with little drop in accuracy. 
}
  \label{fig:sensitivity_plot}
\end{figure*}

Next, we provide results on all datasets in \cref{fig:plot_aime_2024,fig:plot_math500,fig:plot_gsm8k,fig:plot_commonsense_qa,fig:plot_logical_deduction}.

For the O1-Pruner baseline, we ran O1-Pruner on the DeepSeek-Distill-Qwen-7B model using the official code available online and we adapted it to our dataset and model and ran it with the hyperparameters in their codebase. We chose $\lambda=2$ since Table 3 in their paper recommends it. Our method outperforms O1-Pruner by generating models which have higher accuracy at a lower token cost (specifically $\alpha=0.05$ and $\alpha=0.1$). Detailed results are present in Table \ref{tab:normalized_accuracy_tokens}.

\begin{figure*}[h]
\centering
  \includegraphics[width=\textwidth]{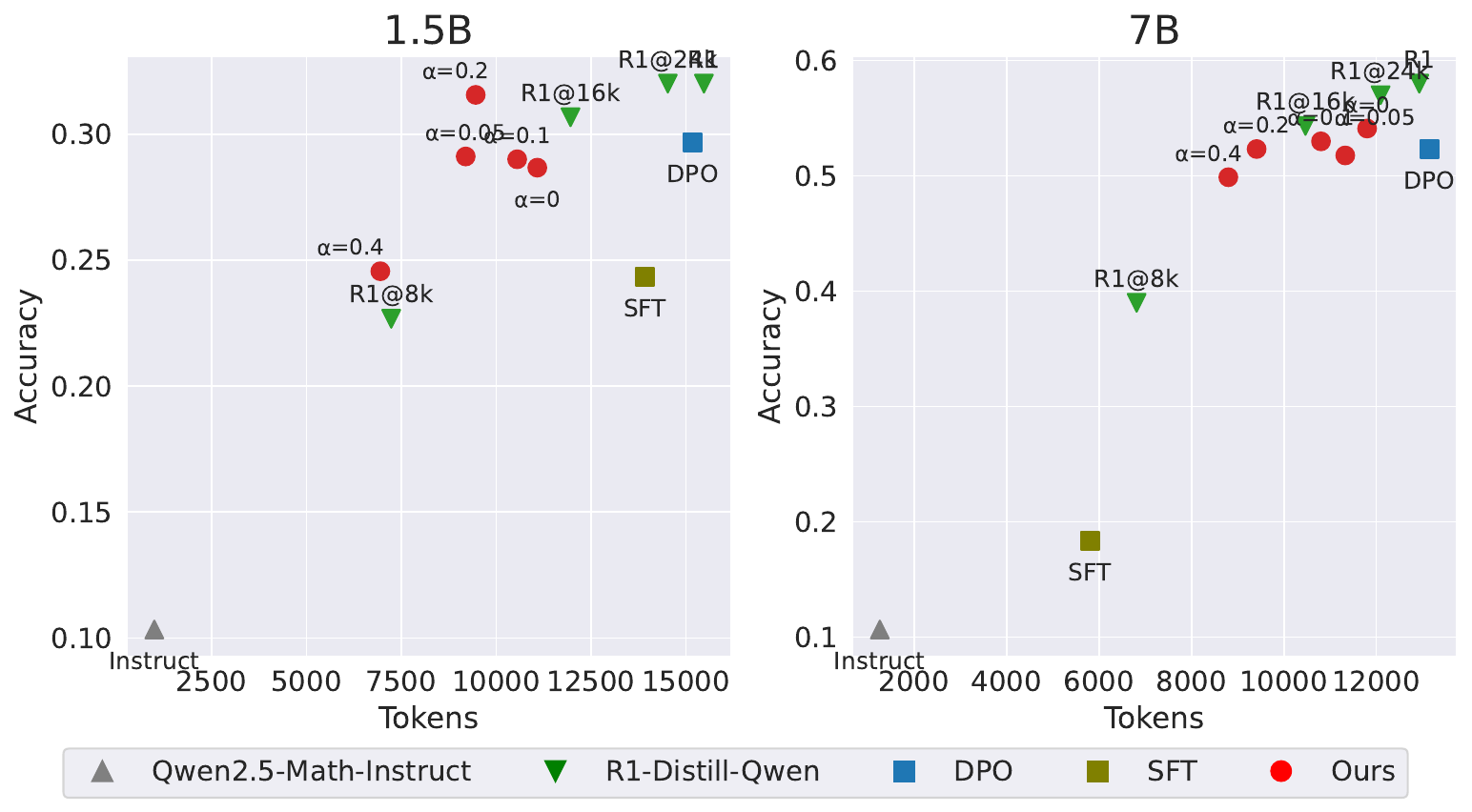}
  \caption{Detailed results on the AIME 2024 dataset.}
  \label{fig:plot_aime_2024}
\end{figure*}

\begin{figure*}[t]
\centering
  \includegraphics[width=\textwidth]{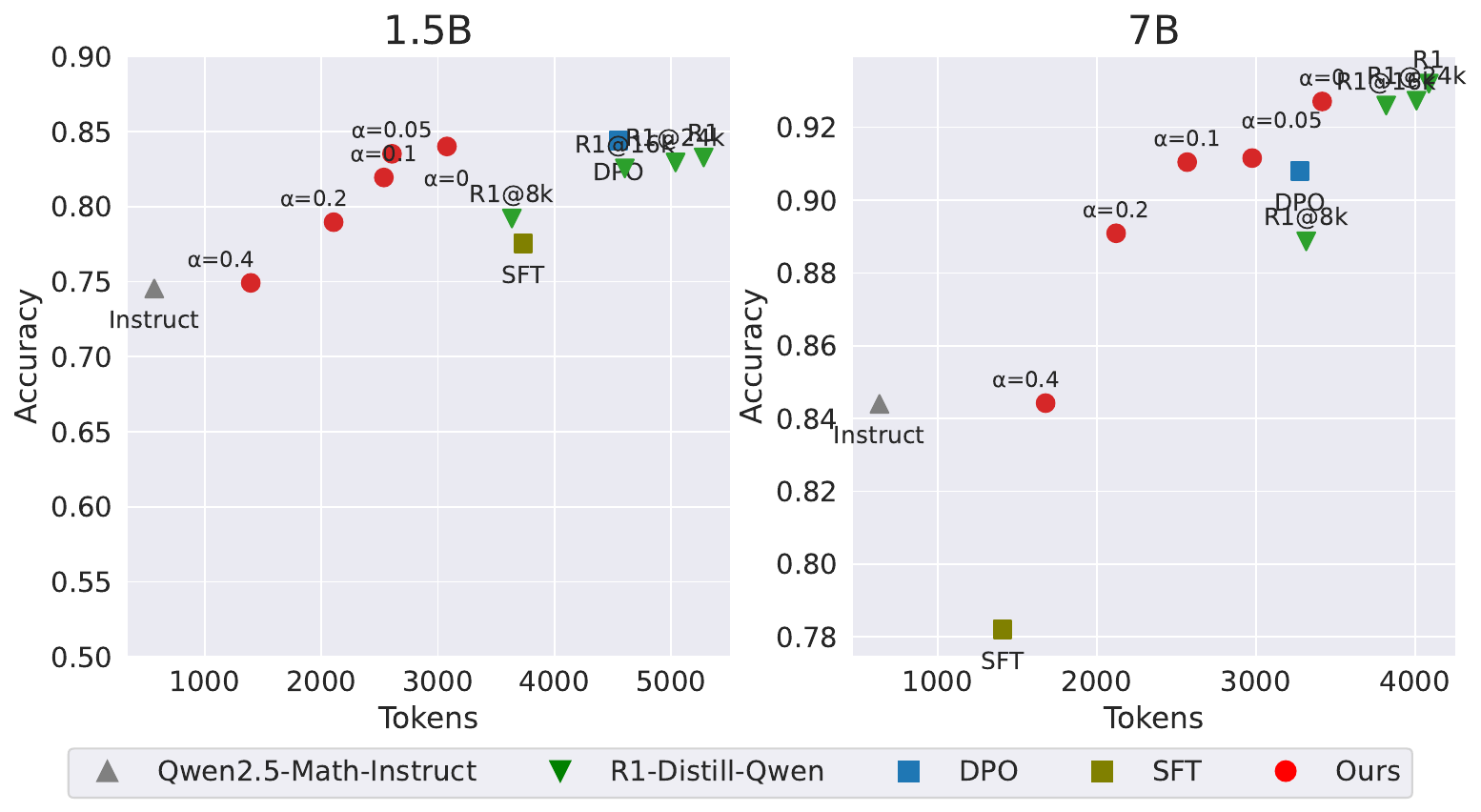}
  \caption{Detailed results on the MATH500 dataset.}
  \label{fig:plot_math500}
\end{figure*}
\begin{figure*}[h!]
\centering
  \includegraphics[width=\textwidth]{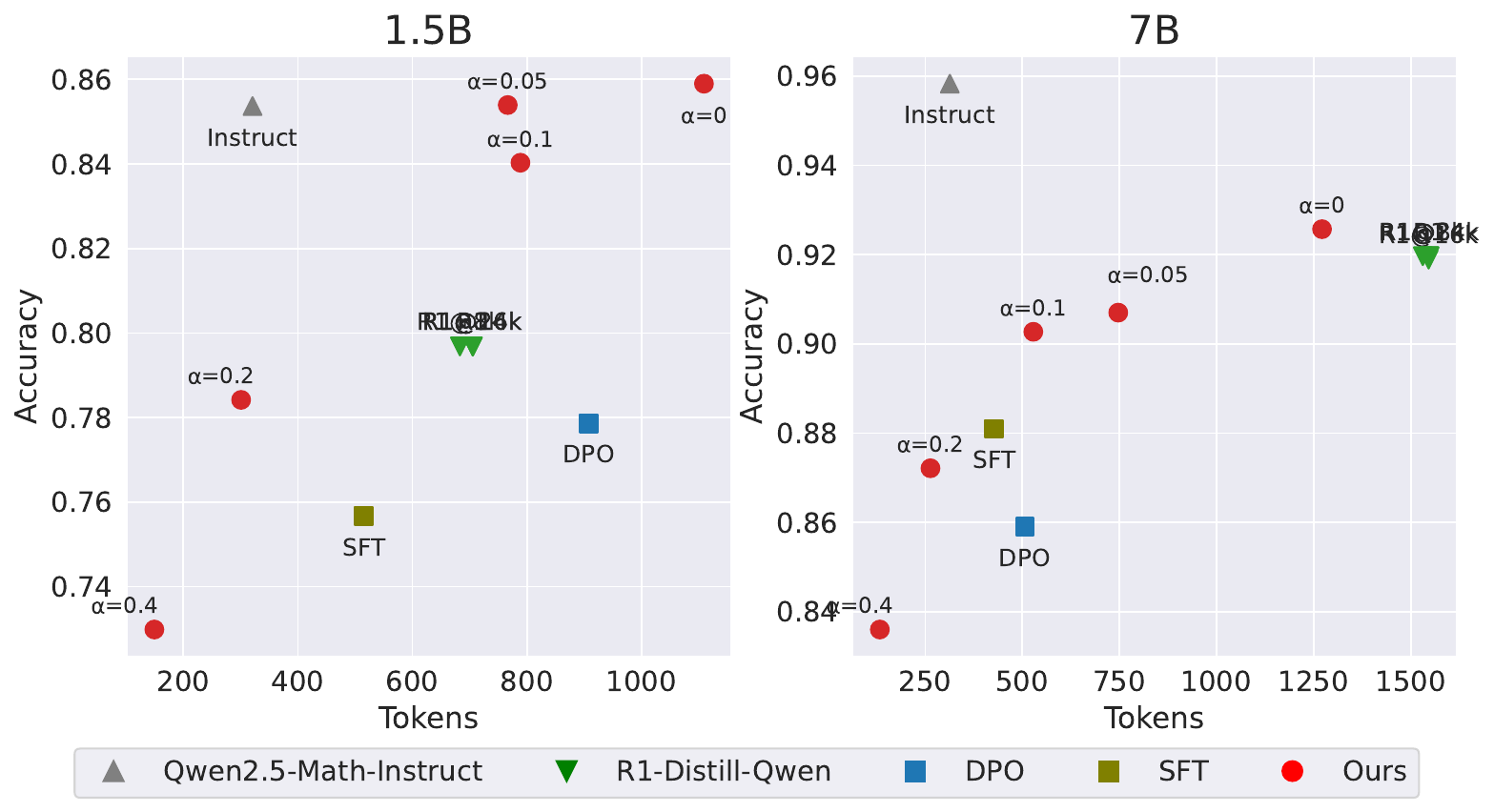}
  \caption{Detailed results on the GSM8K dataset.}
  \label{fig:plot_gsm8k}
\end{figure*}
\begin{figure*}[t]
\centering
  \includegraphics[width=\textwidth]{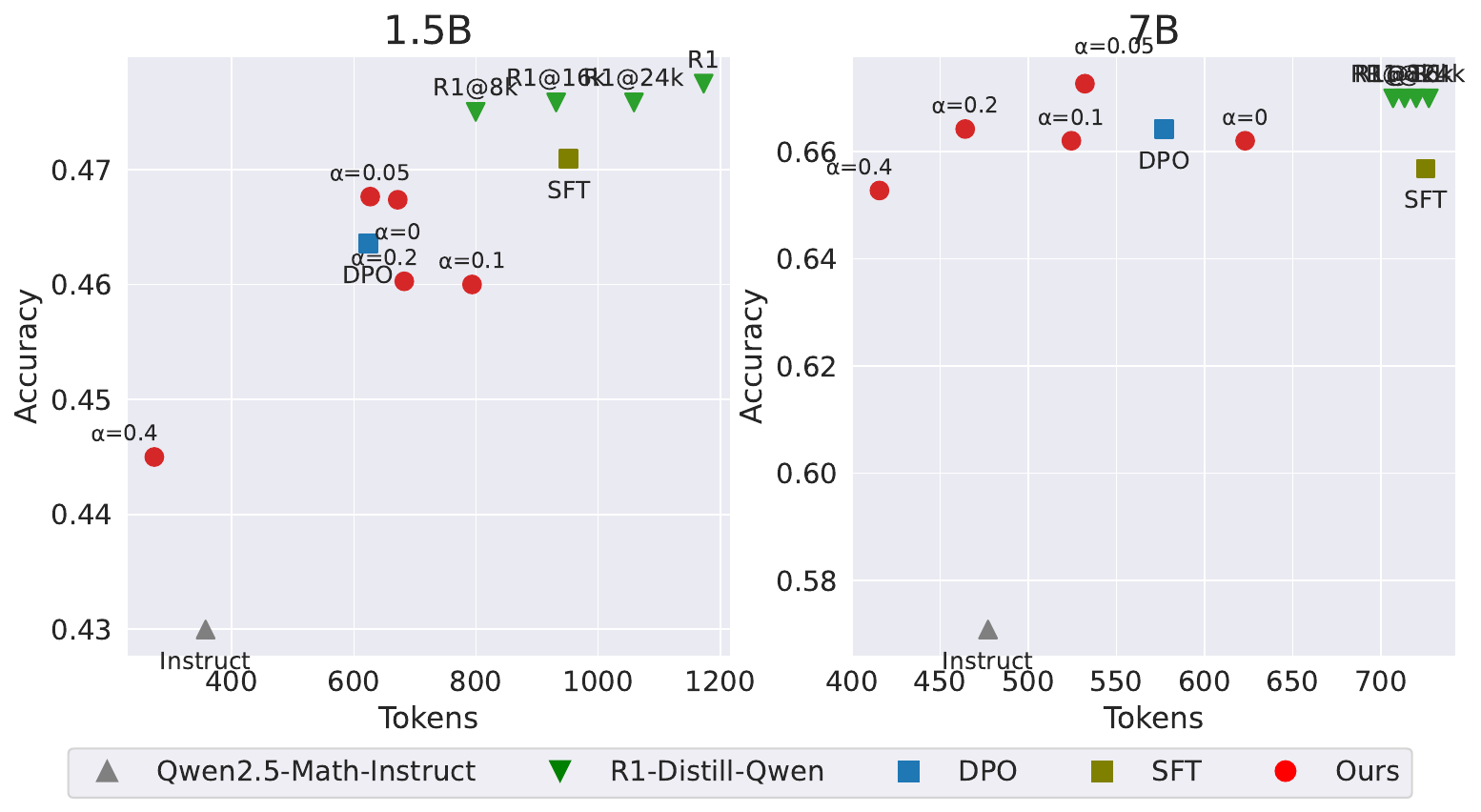}
  \caption{Detailed results on the CommonSenseQA dataset.}
  \label{fig:plot_commonsense_qa}
\end{figure*}
\begin{figure*}[t]
\centering
  \includegraphics[width=\textwidth]{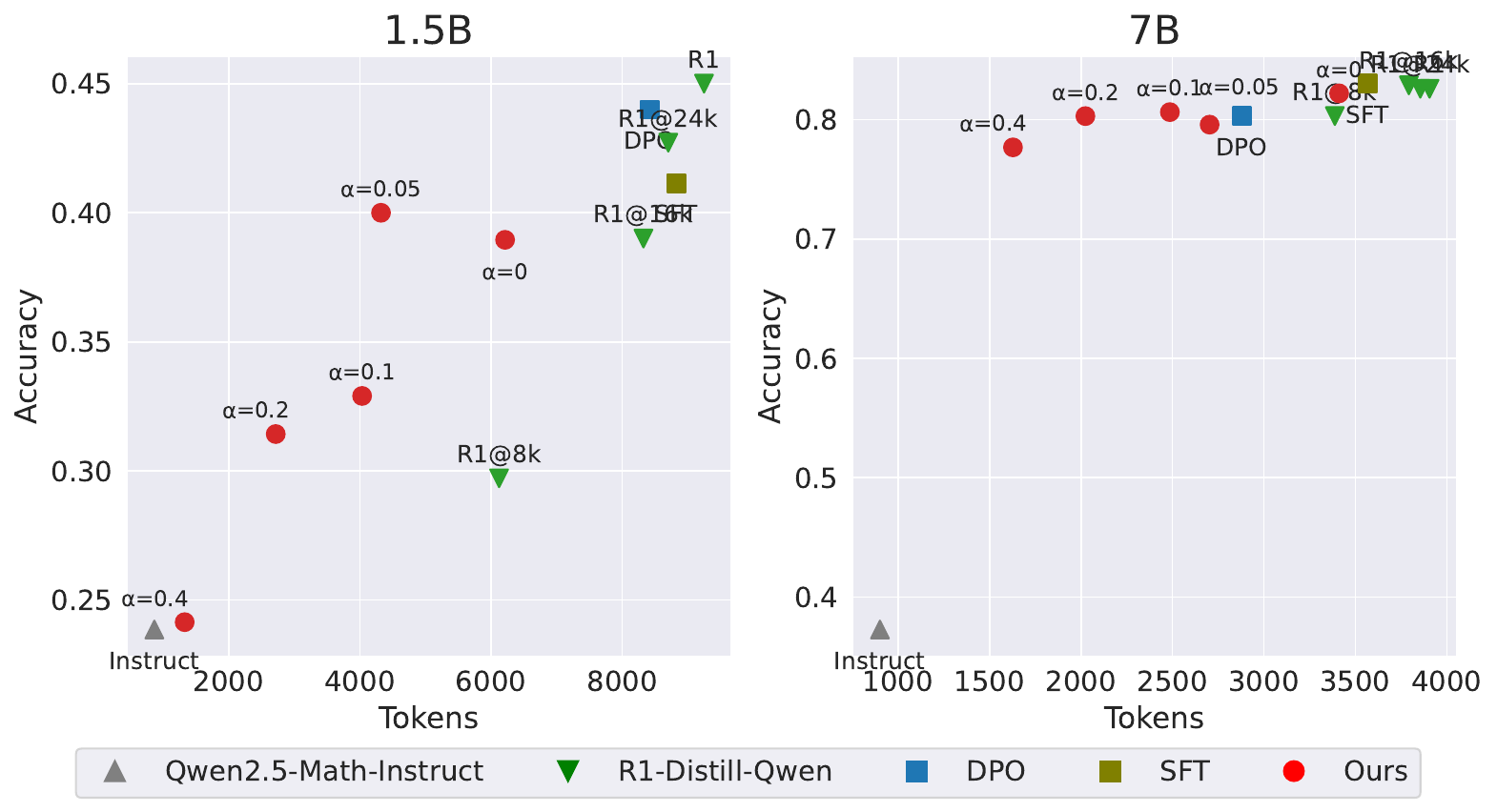}
  \caption{Detailed results on the Logical Deduction dataset.}
  \label{fig:plot_logical_deduction}
\end{figure*}

\label{sec:extended_results}


\clearpage

\begin{table}[h!]
\centering

\begin{tabular}{ccccc}
\hline
\textbf{Model Type} & \textbf{$\alpha$} & \textbf{Tokens} & \textbf{Avg. Pass Rate} & \textbf{Ctx.} \\ \hline
DPO & - & 15184 & 0.297 & 32k \\
Distill & - & 15476 & 0.320 & 32k \\
Distill & - & 14525 & 0.320 & 24k \\
Distill & - & 11960 & 0.307 & 16k \\
Distill & - & 7232 & 0.227 & 8k \\
RL & 0 & 11083 $\pm$ 2022.255   & 0.287 $\pm$ 0.007 & 32k \\
RL & 0.05 & 9196 $\pm$ 897.023 & 0.291 $\pm$ 0.010 & 32k \\
RL & 0.1 & 10552 $\pm$ 1219.273 & 0.290 $ \pm$ 0.000 & 32k \\
RL & 0.2 & 9459 $\pm$ 208.410 & 0.316 $\pm$ 0.037 & 32k \\
RL & 0.4 & 6945 $\pm$ 1661.037 & 0.246 $\pm$ 0.006 & 32k \\
SFT & - & 13922 & 0.243 & 32k\\
\hline
\end{tabular}
\vspace{1em}
\caption{AIME2024 results for 1.5B model}
\label{table:datasets/converted_aime_dataset_1.5B}

\end{table}
\begin{table}[h!]
\centering

\begin{tabular}{ccccc}
\hline
\textbf{Model Type} & \textbf{$\alpha$} & \textbf{Tokens} & \textbf{Avg. Pass Rate} & \textbf{Ctx.} \\ \hline
DPO & - & 4548 & 0.844 & 32k \\
Distill & - & 5278 & 0.833 & 32k \\
Distill & - & 5037 & 0.829 & 24k \\
Distill & - & 4602 & 0.825 & 16k \\
Distill & - & 3634 & 0.792 & 8k \\
RL & 0 & 3077 $\pm$ 478.990   & 0.840 $\pm$ 0.004 & 32k \\
RL & 0.05 & 2606 $\pm$ 189.452 & 0.835 $\pm$ 0.009 & 32k \\
RL & 0.1 & 2536 $\pm$ 162.608 & 0.819 $ \pm$ 0.010 & 32k \\
RL & 0.2 & 2105 $\pm$ 203.024 & 0.790 $\pm$ 0.025 & 32k \\
RL & 0.4 & 1395 $\pm$ 263.378 & 0.749 $\pm$ 0.009 & 32k \\
SFT & - & 3731 & 0.775 & 32k\\
\hline

\end{tabular}
\vspace{1em}
\caption{MATH500 results for 1.5B model}
\label{table:di-zhang-fdu/MATH500_1.5B}

\end{table}
\begin{table}[h!]
\centering

\begin{tabular}{ccccc}
\hline
\textbf{Model Type} & \textbf{$\alpha$} & \textbf{Tokens} & \textbf{Avg. Pass Rate} & \textbf{Ctx.} \\ \hline
DPO & - & 907 & 0.779 & 32k \\
Distill & - & 705 & 0.797 & 32k \\
Distill & - & 705 & 0.797 & 24k \\
Distill & - & 705 & 0.797 & 16k \\
Distill & - & 682 & 0.797 & 8k \\
RL & 0 & 1108 $\pm$ 114.065   & 0.859 $\pm$ 0.008 & 32k \\
RL & 0.05 & 766 $\pm$ 145.017 & 0.854 $\pm$ 0.012 & 32k \\
RL & 0.1 & 788 $\pm$ 130.661 & 0.840 $ \pm$ 0.010 & 32k \\
RL & 0.2 & 300 $\pm$ 80.443 & 0.784 $\pm$ 0.023 & 32k \\
RL & 0.4 & 149 $\pm$ 14.366 & 0.730 $\pm$ 0.036 & 32k \\
SFT & - & 515 & 0.757 & 32k\\
\hline

\end{tabular}
\vspace{1em}
\caption{GSM8K results for 1.5B model}
\label{table:openai/gsm8k_1.5B}

\end{table}
\begin{table}[h!]
\centering
\begin{tabular}{ccccc}
\hline
\textbf{Model Type} & \textbf{$\alpha$} & \textbf{Tokens} & \textbf{Avg. Pass Rate} & \textbf{Ctx.} \\ \hline
DPO & - & 623 & 0.464 & 32k \\
Distill & - & 1172 & 0.477 & 32k \\
Distill & - & 1058 & 0.476 & 24k \\
Distill & - & 931 & 0.476 & 16k \\
Distill & - & 799 & 0.475 & 8k \\
RL & 0 & 671 $\pm$ 67.573   & 0.467 $\pm$ 0.012 & 32k \\
RL & 0.05 & 626 $\pm$ 109.267 & 0.468 $\pm$ 0.013 & 32k \\
RL & 0.1 & 793 $\pm$ 199.640 & 0.460 $ \pm$ 0.009 & 32k \\
RL & 0.2 & 682 $\pm$ 211.393 & 0.460 $\pm$ 0.007 & 32k \\
RL & 0.4 & 273 $\pm$ 77.614 & 0.445 $\pm$ 0.015 & 32k \\
SFT & - & 951 & 0.471 & 32k\\
\hline

\end{tabular}
\vspace{1em}
\caption{CommonSenseQA results for 1.5B model}
\label{table:commonsense_qa_1.5B}
\end{table}
\begin{table}[h!]
\centering

\begin{tabular}{ccccc}
\hline
\textbf{Model Type} & \textbf{$\alpha$} & \textbf{Tokens} & \textbf{Avg. Pass Rate} & \textbf{Ctx.} \\ \hline
DPO & - & 8418 & 0.440 & 32k \\
Distill & - & 9246 & 0.450 & 32k \\
Distill & - & 8701 & 0.427 & 24k \\
Distill & - & 8322 & 0.390 & 16k \\
Distill & - & 6121 & 0.297 & 8k \\
RL & 0 & 6212 $\pm$ 1793.041   & 0.390 $\pm$ 0.092 & 32k \\
RL & 0.05 & 4320 $\pm$ 851.996 & 0.400 $\pm$ 0.027 & 32k \\
RL & 0.1 & 4032 $\pm$ 1464.899 & 0.329 $ \pm$ 0.070 & 32k \\
RL & 0.2 & 2715 $\pm$ 1257.025 & 0.314 $\pm$ 0.053 & 32k \\
RL & 0.4 & 1326 $\pm$ 373.231 & 0.241 $\pm$ 0.016 & 32k \\
SFT & - & 8825 & 0.411 & 32k\\
\hline

\end{tabular}
\vspace{1em}
\caption{Logical Deduction results for 1.5B model}
\label{table:logical_deduction_1.5B}

\end{table}

\begin{table}[h!]

\centering

\begin{tabular}{ccccc}
\hline
\textbf{Model Type} & \textbf{$\alpha$} & \textbf{Tokens} & \textbf{Avg. Pass Rate} & \textbf{Ctx.} \\ \hline
DPO & - & 13147 & 0.523 & 32k \\
Distill & - & 12837 & 0.563 & 32k \\
Distill & - & 12091 & 0.570 & 24k \\
Distill & - & 10467 & 0.543 & 16k \\
Distill & - & 6814 & 0.390 & 8k \\
RL & 0 & 11800 $\pm$ 593.175   & 0.541 $\pm$ 0.026 & 32k \\
RL & 0.05 & 11328 $\pm$ 578.224 & 0.518 $\pm$ 0.014 & 32k \\
RL & 0.1 & 10802 $\pm$ 371.321 & 0.530 $ \pm$ 0.019 & 32k \\
RL & 0.2 & 9410 $\pm$ 84.672 & 0.523 $\pm$ 0.010 & 32k \\
RL & 0.4 & 8798 $\pm$ 811.499 & 0.499 $\pm$ 0.032 & 32k \\
SFT & - & 5806 & 0.183 & 32k\\
\hline
\vspace{1em}
\end{tabular}
\caption{AIME2024 results for 7B model}
\label{table:datasets/converted_aime_dataset_7B}
\end{table}

\begin{table}[h!]
    \centering
\begin{tabular}{ccccc}
\hline
\textbf{Model Type} & \textbf{$\alpha$} & \textbf{Tokens} & \textbf{Avg. Pass Rate} & \textbf{Ctx.} \\ \hline
DPO & - & 3277 & 0.908 & 32k \\
Distill & - & 4086 & 0.932 & 32k \\
Distill & - & 4009 & 0.927 & 24k \\
Distill & - & 3818 & 0.926 & 16k \\
Distill & - & 3316 & 0.889 & 8k \\
RL & 0 & 3416 $\pm$ 192.578   & 0.927 $\pm$ 0.005 & 32k \\
RL & 0.05 & 2976 $\pm$ 143.409 & 0.912 $\pm$ 0.004 & 32k \\
RL & 0.1 & 2568 $\pm$ 206.046 & 0.910 $ \pm$ 0.010 & 32k \\
RL & 0.2 & 2121 $\pm$ 121.491 & 0.891 $\pm$ 0.020 & 32k \\
RL & 0.4 & 1678 $\pm$ 152.337 & 0.844 $\pm$ 0.020 & 32k \\
SFT & - & 1407 & 0.782 & 32k\\
\hline
\end{tabular}
\vspace{1em}
\caption{MATH500 results for 7B model}
\label{table:di-zhang-fdu/MATH500_7B}

\end{table}

\begin{table}[h!]
    \centering
\begin{tabular}{ccccc}
\hline
\textbf{Model Type} & \textbf{$\alpha$} & \textbf{Tokens} & \textbf{Avg. Pass Rate} & \textbf{Ctx.} \\ \hline
DPO & - & 506 & 0.859 & 32k \\
Distill & - & 1547 & 0.920 & 32k \\
Distill & - & 1547 & 0.920 & 24k \\
Distill & - & 1545 & 0.919 & 16k \\
Distill & - & 1530 & 0.920 & 8k \\
RL & 0 & 1271 $\pm$ 72.028   & 0.926 $\pm$ 0.002 & 32k \\
RL & 0.05 & 746 $\pm$ 224.053 & 0.907 $\pm$ 0.011 & 32k \\
RL & 0.1 & 528 $\pm$ 177.573 & 0.903 $ \pm$ 0.008 & 32k \\
RL & 0.2 & 263 $\pm$ 45.282 & 0.872 $\pm$ 0.018 & 32k \\
RL & 0.4 & 133 $\pm$ 30.929 & 0.836 $\pm$ 0.007 & 32k \\
SFT & - & 427 & 0.881 & 32k\\
\hline
\vspace{1em}
\end{tabular}
\caption{GSM8k results for 7B model}
\label{table:openai/gsm8k_7B}
\end{table}

\begin{table}[h!]
    \centering
\begin{tabular}{ccccc}
\hline
\textbf{Model Type} & \textbf{$\alpha$} & \textbf{Tokens} & \textbf{Avg. Pass Rate} & \textbf{Ctx.} \\ \hline
DPO & - & 577 & 0.664 & 32k \\
Distill & - & 727 & 0.670 & 32k \\
Distill & - & 720 & 0.670 & 24k \\
Distill & - & 713 & 0.670 & 16k \\
Distill & - & 706 & 0.670 & 8k \\
RL & 0 & 623 $\pm$ 17.082   & 0.662 $\pm$ 0.012 & 32k \\
RL & 0.05 & 532 $\pm$ 17.618 & 0.673 $\pm$ 0.014 & 32k \\
RL & 0.1 & 524 $\pm$ 44.455 & 0.662 $ \pm$ 0.011 & 32k \\
RL & 0.2 & 464 $\pm$ 9.552 & 0.664 $\pm$ 0.004 & 32k \\
RL & 0.4 & 415 $\pm$ 21.115 & 0.653 $\pm$ 0.007 & 32k \\
SFT & - & 725 & 0.657 & 32k\\
\hline
\vspace{1em}
\end{tabular}
\caption{CommonSenseQA results for 7B model}
\label{table:commonsense_qa_7B}

\end{table}

\begin{table}[h!]
    \centering
\begin{tabular}{ccccc}
\hline
\textbf{Model Type} & \textbf{$\alpha$} & \textbf{Tokens} & \textbf{Avg. Pass Rate} & \textbf{Ctx.} \\ \hline
DPO & - & 2880 & 0.803 & 32k \\
Distill & - & 3905 & 0.826 & 32k \\
Distill & - & 3855 & 0.826 & 24k \\
Distill & - & 3792 & 0.829 & 16k \\
Distill & - & 3388 & 0.803 & 8k \\
RL & 0 & 3411 $\pm$ 292.207   & 0.822 $\pm$ 0.006 & 32k \\
RL & 0.05 & 2703 $\pm$ 224.043 & 0.796 $\pm$ 0.029 & 32k \\
RL & 0.1 & 2485 $\pm$ 262.447 & 0.806 $ \pm$ 0.030 & 32k \\
RL & 0.2 & 2023 $\pm$ 182.110 & 0.803 $\pm$ 0.008 & 32k \\
RL & 0.4 & 1627 $\pm$ 216.580 & 0.777 $\pm$ 0.022 & 32k \\
SFT & - & 3566 & 0.830 & 32k\\
\hline
\vspace{1em}
\end{tabular}
\caption{Logical Deduction results for 7B model}
\label{table:logical_deduction_7B}
\end{table}

\begin{table}[h!]
\centering
\begin{tabular}{lcc}
\toprule
\textbf{Method} & \textbf{Normalized Accuracy} & \textbf{Normalized Tokens} \\
\midrule
R1 & 1.0000 & 1.0000 \\
DPO & 0.9602 & 0.7369 \\
SFT & 0.8216 & 0.5968 \\
$\alpha = 0$ & 0.9891 & 0.8613 \\
$\alpha = 0.05$ & 0.9702 & 0.7034 \\
$\alpha = 0.1$ & 0.9728 & 0.6338 \\
$\alpha = 0.2$ & 0.9594 & 0.5158 \\
$\alpha = 0.4$ & 0.9231 & 0.4341 \\
Instruct & 0.6881 & 0.2685 \\
O1-Pruner & 0.9596 & 0.7713 \\

\hline
\vspace{1em}
\end{tabular}
\caption{Comparison of Normalized Accuracy and Tokens across methods.}
\label{tab:normalized_accuracy_tokens}
\end{table}

\clearpage
\section{Prompt template for training}
\label{sec:prompt_appendix}

For all training purposes, we use the following prompt template: 

\texttt{Please reason step by step, and put your final answer within \textbackslash boxed\{\}. Question: \$QUESTION}

\section{Omitted short proofs}
\label{app:proofs}

\subsection{Proof of \cref{prop:accuracy}}
\begin{proof}
Notice that the objective function 
\cref{eqn:lenpen} can be written as 
\begin{align}
\label{eqn:general}
    \frac{1}{|\mathcal X|}\sum_{\x{}\in \mathcal X}
    \frac{1}{N}\sum_{i \in [N]} p_\theta(\answer_i \mid \x{}) g(\answer_{i})
\end{align}
for a positive function $g(\cdot{})> 0$.
Consider the following lemma.

\begin{lemma}
\label{lem:mass}
For a given prompt $\x{}$, if there exist a correct answer $\answer' = \goldanswer $, then the the population maximizer $\p_\theta$ of \cref{eqn:general} places no mass on the incorrect answers for that prompt, i.e.,
\begin{align}
p_\theta(\answer{} \mid \x{}) = 0, \; \text{if} \;  \answer \neq \goldanswer.
\end{align}
\end{lemma}
\begin{proof}
Suppose the above claim did not hold; in other words, suppose that for some incorrect answer $\answer{} \neq \goldanswer{}$, we have that 
$p_\theta(\answer{} \mid \x{}) > 0$
and that $p_\theta$ maximizes 
\cref{eqn:general}.
Then consider the distribution $\p_{\theta'}$ defined as
\begin{align}
    p_{\theta'}(\answer{} \mid \x{}) & = 0, \; \text{if} \; \answer{} \neq \goldanswer{} \\
    p_{\theta'}(\answer{} \mid \x{}) & \propto p_{\theta}(\answer{} \mid \x{}), \; \text{if} \; \answer{} = \goldanswer{}.
\end{align}
It can be verified that such distribution increases the value of the objective function \eqref{eqn:general} because it places more mass on the positive terms, contradicting the optimality of $\p_\theta$.
\end{proof}

\cref{lem:mass} can be applied to establish the following: if for prompt $\x{}$ there exists a correct answer $\answer = \goldanswer$, then
\begin{align}
    \p_{\theta^\star_{eff}}(\answer{} \mid \x{}) & = 0, \; \text{if} \;  \answer \neq \goldanswer 
\end{align}
which implies that both $\p_{\theta^\star_{eff}}$ 
has its support on the correct answers only, proving the claim.
\end{proof}

\subsection{Proof of \cref{prop:optimality}}

\begin{proof}



Assume the contrapositive, that is $\exists \; \x \in \mathcal{X}, \exists \; \y'$ such that $\y'=\y^*(x)$
\[ 
\E_{\y\sim p_{\theta^*_{eff}(x)}}\big[\len(y)\big] > \len(\y') \] 
Then consider the modified distribution $p_{\theta'}$ which places all mass on the response $\y'$, that is, 
\[
p_{\theta'}(y|x) = \1\{y=y'\}
\]
then
\[
\E_{\y\sim p_{\theta^*_{eff}}(x)}
\big[\len(y)\big] > \len(\y') = \E_{\y\sim p_{\theta'}(x)} \big[\len(y)\big]
\]
Since $f$ is a monotonically increasing function bounded in $[0, 1]$ and $\alpha \in [0, 1)$
\[
\E_{\y\sim p_{\theta^*_{eff}}(x)}
\big[1 - \coeff \cdot f(\len(y))\big] < \E_{\y\sim p_{\theta'}(x)} \big[1 - \coeff \cdot f(\len(y))\big]
\]
Since $\accuracy(p_{\theta^*_{eff}})=\accuracy(p_{\theta'})=1$, we have that 

\[
\Rightarrow \E_{\y\sim p_{\theta^*_{eff}}(x)}
\big[\1\{y=y^*(x)\}(1 - \coeff \cdot f(\len(y))\big]) < \E_{\y\sim p_{\theta'}(x)} \big[\1\{y=y^*(x)\}(1 - \coeff \cdot f(\len(y)))\big]
\]

This would imply that $\theta^*_{eff}$ is not the population-level maximizer which is a contradiction.


    
\end{proof}
\clearpage
\section{Do Distilled Models follow Length Constraints?}
\label{sec:length_constraints}
One simple way to improve efficiency of reasoning models would be to simply prompt with an instruction saying: `Respond in less than X tokens'. This would prevent very verbose responses from the model. However, we perform an experiment to check whether the distilled models even have such a capability. For the MATH500 test set, we prompt the model in the following manner: 

\texttt{Please think step by step and answer in less than X tokens.}

\texttt{Question: \{question\}}

\texttt{Answer:}

We vary X in the range 256, 512, 1024, 2048, 4096 and measure the number of output tokens. The results for the 1.5B and 7B Distilled models are in Tables \ref{tab:length_constraints_1.5B} and \ref{tab:length_constraints_7B}. The results clearly demonstrate that their is no correlation between the requested number of tokens and the number of tokens that are actually generated by the model. This leads us to believe that the Distilled models are not capable of following length constraints out-of-the-box.

\begin{table}[h!]
\centering
\begin{tabular}{c|c}
\textbf{Token Limit} & \textbf{Tokens Generated} \\
\hline
256  & 4609.34 \\
512  & 4915.71 \\
768  & 5228.85 \\
1024 & 4913.84 \\
1280 & 5306.68 \\
2048 & 5064.06 \\
4096 & 5245.11 \\
\end{tabular}
\vspace{1em}

\caption{This table lists the average number of tokens generated by Distilled-R1-Qwen-1.5B for varying token limits as mentioned in the prompt for the MATH500 test set.}
\label{tab:length_constraints_1.5B}
\end{table}

\begin{table}[h!]
\centering
\begin{tabular}{c|c}
\textbf{Token Limit} & \textbf{Tokens Generated} \\
\hline
256  & 3434.56 \\
512  & 3587.05 \\
768  & 3518.34 \\
1024 & 3716.17 \\
1280 & 3524.46 \\
2048 & 3688.01 \\
4096 & 3815.11 \\
\end{tabular}
\vspace{1em}
\caption{This table lists the average number of tokens generated by Distilled-R1-Qwen-7B for varying token limits as mentioned in the prompt for the MATH500 test set.}
\label{tab:length_constraints_7B}
\end{table}




\clearpage
\section{Length Reduction without Length Penalty?}
\label{sec:alpha_0}

Reduction in length when $\alpha=0$ is an intriguing observation. Recent work by \cite{liu2025understandingr1zeroliketrainingcritical} pointed out a bias in the GRPO loss function: it averages per-token loss across entire sequences, which unintentionally favors shorter correct sequences over longer correct ones, and longer incorrect sequences over shorter incorrect ones. Another inadvertent issue that this creates is that it increases loss weightage on problems which are easier since they generally have shorter Chain-of-Thoughts. This may explain the unexpected reduction in reasoning length, even when $\alpha=0$. We tested the fix proposed by \cite{liu2025understandingr1zeroliketrainingcritical} and observed that the length reduction disappears when the fix is applied. Table \ref{tab:dr_grpo_fix} below shows changes in normalized accuracy and token usage (relative to a baseline 7B distilled model). The average results highlight that the fix mitigates the unintended length bias.

\begin{table}[h!]
\centering
\begin{tabular}{lcccccc}
\toprule
Dataset & \makecell{RLOO+Fix \\($\Delta$ NT)} & \makecell{RLOO+Fix \\($\Delta$ NA)} & \makecell{RLOO \\($\Delta$ NT)} & \makecell {RLOO \\ ($\Delta$ NA)} & \makecell{ Baseline \\(NT) } & \makecell{Baseline \\ (NA)} \\
\midrule
MATH500   & 2.3   & -0.4  & -17.4  & -0.6  & 100 & 100 \\
AIME2024  & 8     & -3    & -10.9  & -3.6  & 100 & 100 \\
GSM8k     & -12.2 & -3.37 & -17.2  & 1.08  & 100 & 100 \\
\midrule
Average   & \textbf{-0.64} & -2.25 & -15.16 & -1.04 & 100 & 100 \\
\bottomrule
\vspace{1em}
\end{tabular}
\caption{Table showing the effects of fixes proposed by \cite{liu2025understandingr1zeroliketrainingcritical} $\Delta$ NT refers to change in normalized tokens. $\Delta$ NA refers to change in normalized accuracy. All numbers are normalized based on the Baseline scores. All experiments have been conducted on the 7B Distilled model.}
\label{tab:dr_grpo_fix}
\end{table}

\section{Training dynamics}
\label{sec:training_dynamics}

We present the performance on the training dataset in \cref{fig:training_dynamics}. Notably, setting $\coeff = 0$ corresponds to applying RL without any length penalty. Increasing $\coeff$ results in a significant reduction in token usage—up to 50\% compared to the initial model—while maintaining the same level of accuracy as at the beginning of RL training. Lower values of $\coeff$ improve performance while still reducing the number of tokens.


\begin{figure*}[h!]
\centering
\includegraphics[width=\textwidth]{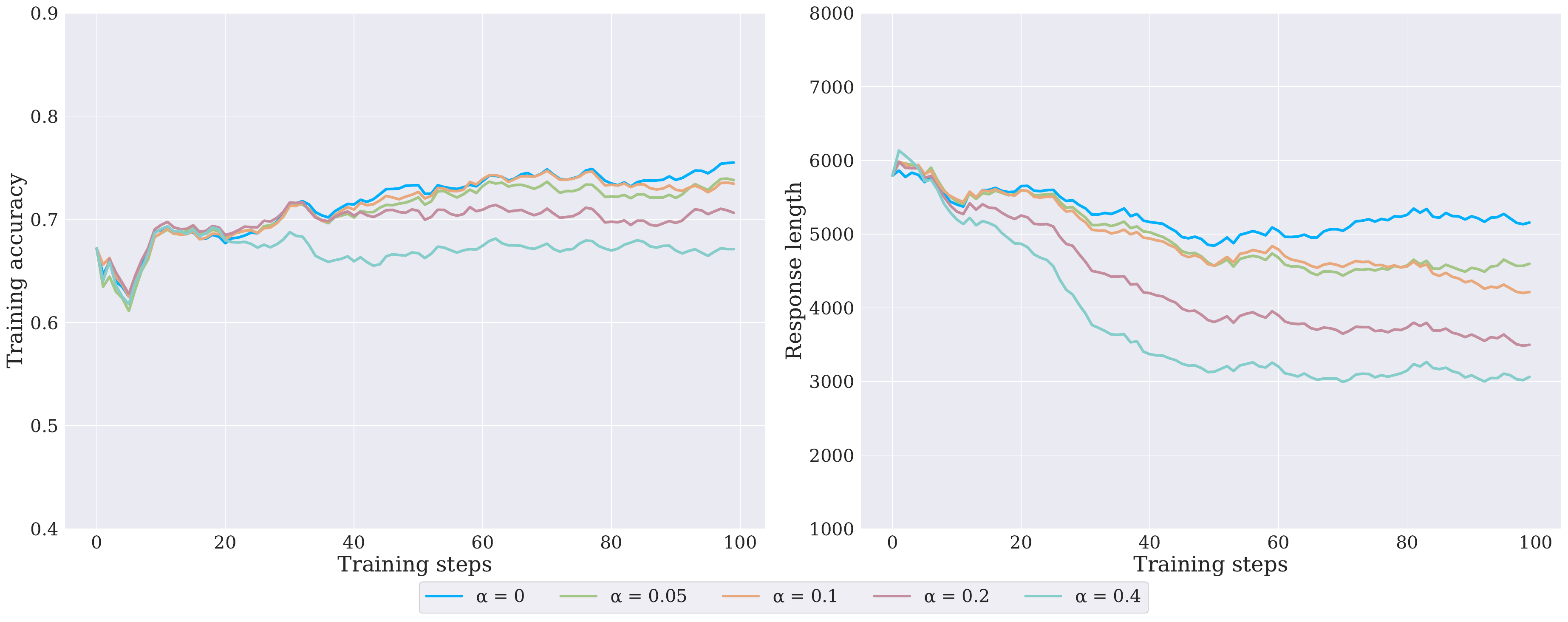}
\caption{The figure shows the dynamics of the training accuracy and the corresponding generation lengths with varying values of $\alpha$ for the 7B model. The training accuracy and response length have been smoothed out using running averages over 25 training iterations.} 
\label{fig:training_dynamics}
\end{figure*}

\clearpage

\section{Difficulty based analysis}
\label{sec:difficulty_analysis}

As discussed in the Results \ref{sec:results} earlier, we observe a bigger reduction in easier datasets like GSM8K as compared to harder datasets such as AIME2024. We hypothesize that this is because for easier datasets, reasoning models have a larger relative `spread'. 
To study this rigorously, we compute the quantities: $\frac{\textsc{STD}(x)}{\textsc{MEAN}(x)}$ and $\frac{\textsc{MEAN}(x)-\textsc{MIN}(x)}{\textsc{MEAN}(x)}$ for GSM8K, MATH500 and AIME2024 datasets for the Distilled 7B model.  The values we get are reported in Table \ref{tab:relative_spread}. As we can see that the relative spread is much larger for the easier dataset (GSM8K) compared to the hardest dataset (AIME2024). This clearly demonstrates that reasoning models waste relatively more tokens thinking about easier problems compared to hard problems.


\begin{table}[t]
    \centering
    
\begin{tabular}{c|cc}
\toprule
Dataset & $\frac{\textsc{STD}(x)}{\textsc{MEAN}(x)}$  & $\frac{\textsc{MEAN}(x)-\textsc{MIN}(x)}{\textsc{MEAN}(x)}$  \\
\midrule
GSM8K & 0.357 & 0.437 \\
MATH500 & 0.271 & 0.323 \\
AIME & 0.254 & 0.264 
\vspace{1em}
\end{tabular}
\label{tab:relative_spread}
\caption{Relative spread for different datasets for the Distilled 7B model. We only aggregate over prompts where there is atleast 1 correct response out of 8 generations.}
\end{table}
\vspace{3em}

\section{Length penalty when training a base model.}
\label{sec:base_model_expt}

\begin{figure*}[h!]
\centering
\includegraphics[width=\textwidth]{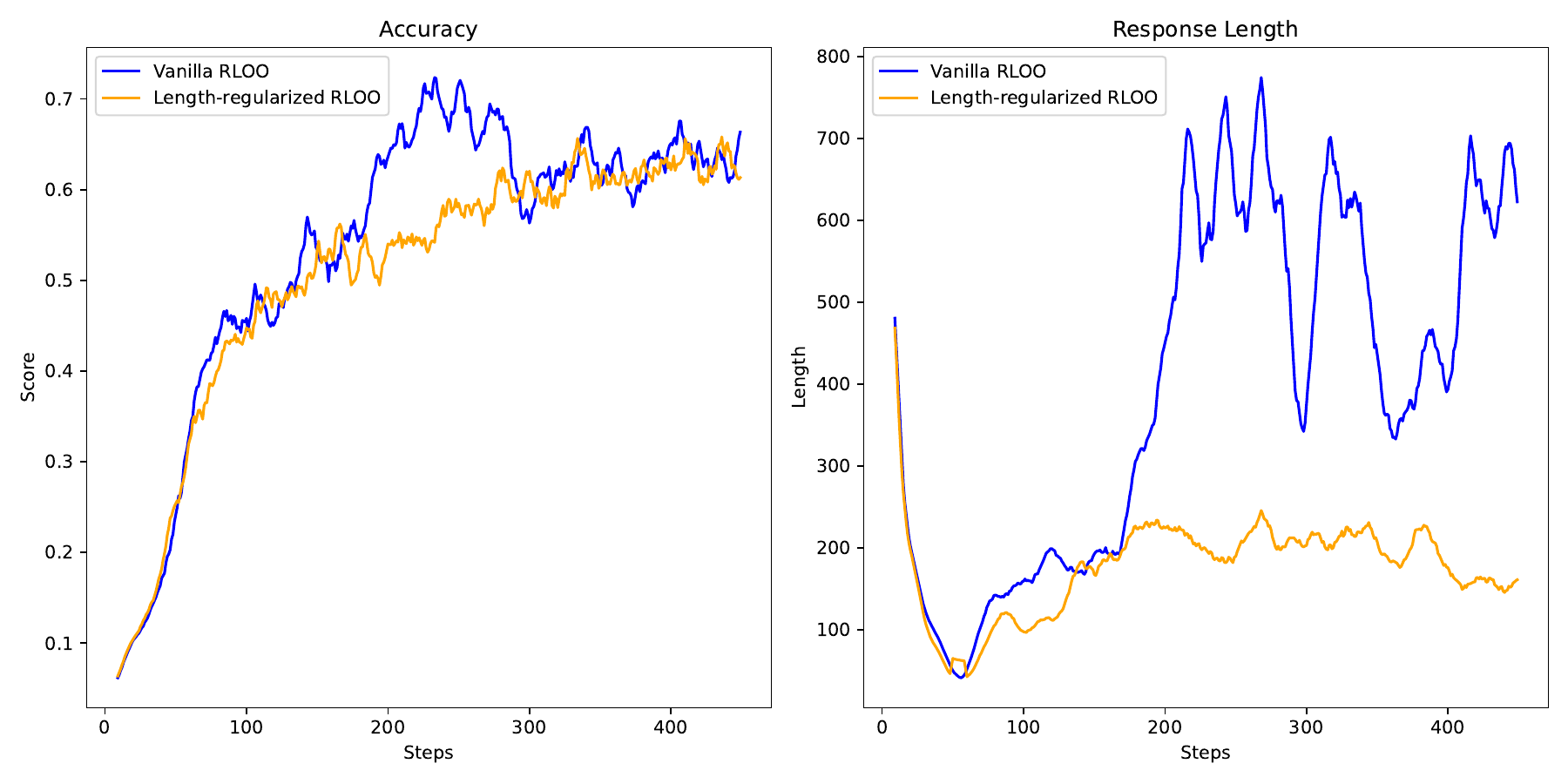}
\caption{The figure shows the dynamics of the training accuracy and the corresponding generation lengths for vanilla and length regularized RL on a base model for the Countdown task.} 
\label{fig:length_and_scores}
\end{figure*}

Base models such as Qwen or Llama generally do not exhibit long CoTs before RLVR training. However, as training progresses, their response length and accuracy increase. In general, training a base model for reasoning is a very compute-intensive task \citep{liu2025prorlprolongedreinforcementlearning}. However, we do a preliminary investigation of the effect of length-regularized training on a small-scale replication of long CoT behaviour in the Countdown task as done in TinyZero by \cite{tinyzero}.

We use the Qwen2.5-3B base model and run RLOO with and without length penalty (with $\alpha=0.1$) with a context window limit of 2048 tokens. The training curves can be seen in Figure \ref{fig:length_and_scores}. We make the following observations as training progresses:

\begin{itemize}
    \item \textbf{Convergence is slower when training using the length penalty}

    Vanilla RLOO reaches peak performance at about 220 RL steps, after which it saturates due to context window limits whereas our method reaches peak performance at about 400 RL steps.

    \item \textbf{Vanilla RLOO reaches higher peak performance}

    Vanilla RLOO reaches a peak performance of 67.48\% on the validation set whereas length-regularization reaches a peak performance of 59.68%

    \item \textbf{During training, when controlled for performance, length penalized training offers a more efficient model}

    As training progresses and performance improves, vanilla RLOO’s response length grows significantly. However, for the same training accuracy in length-regularized RLOO, the corresponding response length is significantly lower. Notice that we have only tried one combination of hyperparameters.

Results are given in Table \ref{tab:rloo_lengths}

\end{itemize}

\begin{table}[h!]
\centering
\begin{tabular}{lcc}
\toprule
\textbf{Training accuracy} &
\makecell{\textbf{Vanilla}\\\textbf{RLOO Response}\\\textbf{Length}} &
\makecell{\textbf{Length-regularized}\\\textbf{RLOO Response}\\\textbf{Length}} \\
\midrule
50\% & 271.23 & 222.11 \\
53\% & 297.43 & 203.59 \\
56\% & 373.45 & 215.89 \\
59\% & 432.12 & 138.78 \\
\bottomrule
\vspace{1em}
\end{tabular}
\caption{Comparison of response lengths for Vanilla vs. Length-regularized RLOO models.}
\label{tab:rloo_lengths}
\end{table}

We also observe syntactic differences in the responses learnt by both models:

For instance, an example response for the vanilla RLOO model was:

\textbf{Prompt:}
\texttt{Using the numbers [71, 3, 72], create an equation that equals 95. You can use basic arithmetic operations (+, -, *, /) and each number can only be used once. Show your work in <think> </think> tags. And return the final answer in <answer> </answer> tags, for example <answer> (1 + 2) / 3 </answer>.}

\textbf{Vanilla RLOO response}
\texttt{Let me solve this step by step. <think>72 + 71 = 143, then 143 - 3 = 140, but this doesn't work. Try: 3 + 72 = 75, then 75 + 71 = 146, but this doesn't work. Try: 71 * 3 = 213, then 213 - 72 = 141, but this doesn't work. Try: 71 + 72 = 143, then 143 - 3 = 140, but this doesn't work. Try: 72 * 3 = 216, then 216 - 71 = 145, but this doesn't work. Try: 3 * 72 = 216, then 216 - 71 = 145, but this doesn't work. Try: 3 + 72 = 75, then 75 * 71 = 5425, but this doesn't work. Try: 71 / 3 = 23.(3), then 23.(3) + 72 = 95.73, but this doesn't work. Try: 71 - 72 = -1, then -1 + 3 = 2, but this doesn't work. Try: 72 / 71 = 1.(03), then 1.(03) + 72 = 73.(03), but this doesn't work. Try: 72 - 71 = 1, then 1 + 3 = 4, but this doesn't work. Try: 3 * 71 = 213, then 213 + 72 = 285, but this doesn't work. Try: 71 - 72 = -1, then -1 * 3 = -3, but this doesn't work. Try: 3 * 72 = 216, then 216 + 71 = 287, but this doesn't work. Try: 72 - (71 * 3) = 72 - 213 = -141, but this doesn't work. Try: (72 + 71) * 3 = 143 * 3 = 430, but this doesn't work. Try: 71 + (72 * 3) = 71 + 216 = 287, but this doesn't work. Try: 71 - (72 / 3) = 71 - 24 = 47, but this doesn't work. Try: 72 - (71 / 3) = 72 - 23.(3) = 48.(7), but this doesn't work. Try: 72 + 71 / 3 = 72 + 23.(3) = 95.(3), but this doesn't work. Try: 71 - (72 - 3) = 71 - 75 = -4, but this doesn't work. Try: 71 + 72 - 3 = 143 - 3 = 140, but this doesn't work. Try: 72 / 71 * 3 = 3 * (72 / 71) = 3 * 1.(03) = 3.(09), but this doesn't work. Finally: 3 * (71 + 72) = 3 * 143 = 43.(9), but this doesn't work. Try: 71 - 72 + 72 * 3 = -1 + 216 = 215, but this doesn't work. Ah-ha! 72 - 71 + 72 * 3 = 1 + 216 = 217, but this doesn't work. One more: 3 + 71 * 72 = 3 + 5112 = 5115, but this doesn't work. There goes: (71 + 72) / 3 = 143 / 3 = 47.(6), but this doesn't work. One more: 72 - 71 - 72 / 3 = 1 - 24 = -23, but this doesn't work. There is none: 72 / (71 - 72) = 72 / 1 = 72, but this doesn't work. Which one finally: 72 - (71 + 72 / 3) = 72 - (71 + 24) = 72 - 95 = -23, but this doesn't work. Think deep: 71 - 72 + 72 / 3 = -1 + 24 = 23, but this doesn't work. Finally: <answer>71 - 72 / 3 + 72</answer></thought>}

However, for the length regularized RLOO variant, we observe the following

\textbf{Prompt:}
\texttt{Using the numbers [32, 45, 76, 54], create an equation that equals 53. You can use basic arithmetic operations (+, -, *, /) and each number can only be used once. Show your work in <think> </think> tags. And return the final answer in <answer> </answer> tags, for example <answer> (1 + 2) / 3 </answer>.}

\textbf{Length regularized RLOO response}
\texttt{Let me solve this step by step. <think> 76-54=22 22+45=67 67-32=35 Nope Try: 54+45=99 99-76=23 23+32=55 Nope Try: 76+54=130 130-45=85 85-32=53 This works! </th>
<answer>76+54-45-32</answer>}

Our results in this toy setup demonstrate that even though length-regularization slows down convergence and reduces peak performance, it discovers more token-efficient solutions when controlled for performance. Due to this, we believe that our method, as intended in the paper, \textbf{would be more effective when used as a post training method compared to doing RLVR from scratch on a base model.}

\pagebreak

%% file: checklist.tex
\section*{NeurIPS Paper Checklist}

\begin{enumerate}

\item {\bf Claims}
    \item[] Question: Do the main claims made in the abstract and introduction accurately reflect the paper's contributions and scope?
    \item[] Answer: \answerYes 
    \item[] Justification: The experiments demonstrate that our approach allows the derivation of a family of models on the accuracy-compute tradeoff curve.

\item {\bf Limitations}
    \item[] Question: Does the paper discuss the limitations of the work performed by the authors?
    \item[] Answer: \answerYes{} 
    \item[] Justification: Yes, see section \ref{sec:limitations}

\item {\bf Theory assumptions and proofs}
    \item[] Question: For each theoretical result, does the paper provide the full set of assumptions and a complete (and correct) proof?
    \item[] Answer: \answerYes{} 
    \item[] Justification: See \cref{app:proofs} and \cref{sec:assumptions_and_props}

    \item {\bf Experimental result reproducibility}
    \item[] Question: Does the paper fully disclose all the information needed to reproduce the main experimental results of the paper to the extent that it affects the main claims and/or conclusions of the paper (regardless of whether the code and data are provided or not)?
    \item[] Answer: \answerYes{} 
    \item[] Justification: 
    See \cref{app:implementation_details}. Code and data will also be provided.

\item {\bf Open access to data and code}
    \item[] Question: Does the paper provide open access to the data and code, with sufficient instructions to faithfully reproduce the main experimental results, as described in supplemental material?
    \item[] Answer: \answerNA{} 
    \item[] Justification: 

\item {\bf Experimental setting/details}
    \item[] Question: Does the paper specify all the training and test details (e.g., data splits, hyperparameters, how they were chosen, type of optimizer, etc.) necessary to understand the results?
    \item[] Answer: \answerYes{} 
    \item[] Justification:  See \cref{sec:experiments} and \cref{app:implementation_details}. 

\item {\bf Experiment statistical significance}
    \item[] Question: Does the paper report error bars suitably and correctly defined or other appropriate information about the statistical significance of the experiments?
    \item[] Answer: \answerYes{} 
    \item[] Justification: See \cref{fig:sensitivity_plot}

\item {\bf Experiments compute resources}
    \item[] Question: For each experiment, does the paper provide sufficient information on the computer resources (type of compute workers, memory, time of execution) needed to reproduce the experiments?
    \item[] Answer: \answerYes{} 
    \item[] Justification: 
    See Section \cref{app:implementation_details}
    
\item {\bf Code of ethics}
    \item[] Question: Does the research conducted in the paper conform, in every respect, with the NeurIPS Code of Ethics \url{https://neurips.cc/public/EthicsGuidelines}?
    \item[] Answer: \answerYes 
    \item[] Justification: 

\item {\bf Broader impacts}
    \item[] Question: Does the paper discuss both potential positive societal impacts and negative societal impacts of the work performed?
    \item[] Answer: \answerYes{} 
    \item[] Justification: See \cref{sec:introduction}
    
\item {\bf Safeguards}
    \item[] Question: Does the paper describe safeguards that have been put in place for responsible release of data or models that have a high risk for misuse (e.g., pretrained language models, image generators, or scraped datasets)?
    \item[] Answer: \answerNA{} 
    \item[] Justification: 

\item {\bf Licenses for existing assets}
    \item[] Question: Are the creators or original owners of assets (e.g., code, data, models), used in the paper, properly credited and are the license and terms of use explicitly mentioned and properly respected?
    \item[] Answer: \answerYes{} 
    \item[] Justification: All models and datasets have been cited.

\item {\bf New assets}
    \item[] Question: Are new assets introduced in the paper well documented and is the documentation provided alongside the assets?
    \item[] Answer: \answerNA{} 
    \item[] Justification: 

\item {\bf Crowdsourcing and research with human subjects}
    \item[] Question: For crowdsourcing experiments and research with human subjects, does the paper include the full text of instructions given to participants and screenshots, if applicable, as well as details about compensation (if any)? 
    \item[] Answer: \answerNA{} 
    \item[] Justification: 

\item {\bf Institutional review board (IRB) approvals or equivalent for research with human subjects}
    \item[] Question: Does the paper describe potential risks incurred by study participants, whether such risks were disclosed to the subjects, and whether Institutional Review Board (IRB) approvals (or an equivalent approval/review based on the requirements of your country or institution) were obtained?
    \item[] Answer: \answerNA{} 
    \item[] Justification: 

\item {\bf Declaration of LLM usage}
    \item[] Question: Does the paper describe the usage of LLMs if it is an important, original, or non-standard component of the core methods in this research? Note that if the LLM is used only for writing, editing, or formatting purposes and does not impact the core methodology, scientific rigorousness, or originality of the research, declaration is not required.
    \item[] Answer: \answerNA{} 
    \item[] Justification: 

\end{enumerate}